\documentclass[sigconf]{acmart}

\AtBeginDocument{%
  }

\usepackage{algorithm}
\usepackage{algpseudocode}
\usepackage{multirow} 
\usepackage{booktabs}
\usepackage{siunitx}
\usepackage[table]{xcolor}
\usepackage{subcaption}

\copyrightyear{2026}
\acmYear{2026}
\setcopyright{cc}
\setcctype{by}
\acmConference[KDD '26]{Proceedings of the 32nd ACM SIGKDD Conference on Knowledge Discovery and Data Mining V.2}{August 09--13, 2026}{Jeju Island, Republic of Korea}
\acmBooktitle{Proceedings of the 32nd ACM SIGKDD Conference on Knowledge Discovery and Data Mining V.2 (KDD '26), August 09--13, 2026, Jeju Island, Republic of Korea}
\acmDOI{10.1145/3770855.3817925}
\acmISBN{979-8-4007-2259-2/2026/08}

\providecommand{\orcidsite}[1]{#1}

\begin{document}

\title{On the Push-Based Asynchronous Federated Learning: A Bias-Correction Aggregation Approach}

\author{Jiahui Bai}
\affiliation{%
  \institution{School of Computer Technologies, RMIT University}
  \city{Melbourne}
  \state{VIC}
  \country{Australia}
  \orcid{0009-0000-8304-5455}
}
\email{jiahui.bai2@student.rmit.edu.au}

\author{Hai Dong}
\authornote{Corresponding author.}
\affiliation{%
  \institution{School of Computer Technologies, RMIT University}
  \city{Melbourne}
  \state{VIC}
  \country{Australia}
  \orcid{0000-0002-7033-5688}
}
\email{hai.dong@rmit.edu.au}

\author{A. K. Qin}
\affiliation{%
  \institution{School of Science, Computing and Engineering Technologies, Swinburne University of Technology}
  \city{Hawthorn}
  \state{VIC}
  \country{Australia}
  \orcid{0000-0001-6631-1651}
}
\email{kqin@swin.edu.au}

\renewcommand{\shortauthors}{Bai et al.}

\begin{abstract}
Asynchronous decentralized federated learning (ADFL) eliminates central coordination and global synchronization, making it attractive for large-scale and heterogeneous systems. However, frequent peer-to-peer communication, asynchronous updates on directed topologies, and non-IID data jointly lead to excessive communication overhead, biased aggregation and severe model drift.
We propose PushCen-ADFL, a communication-efficient ADFL framework that enables stable training under asymmetric communication and delayed client participation. 
PushCen-ADFL couples communication, aggregation, and local stabilization in a shared centroid representation space, forming a closed loop between compression and optimization.
Clients exchange centroid-form messages, apply average-preserving push-sum mixing to correct aggregation bias, and use a lightweight centroid regularization anchored in the same centroid space to mitigate drift under heterogeneity and staleness.
A bounded, sender-deduplicated buffer further improves robustness under irregular asynchronous arrivals.
Experiments on vision datasets demonstrate that PushCen-ADFL improves accuracy under data heterogeneity by up to 6\% while reducing per-push communication cost by more than 80\%, achieving a favorable accuracy-communication trade-off.

\end{abstract}



\maketitle

\section{Introduction}

Federated Learning (FL) has emerged as a privacy-preserving distributed machine learning paradigm that enables collaborative model training when data cannot be centrally shared \cite{mcmahan2017communication, 8270639, mughal2024adaptive}. By performing training on local devices and exchanging only model updates, FL allows distributed data resources to be effectively utilized without direct access to raw data and has demonstrated significant potential in applications such as mobile intelligence, the Internet of Things and cross-organizational collaboration \cite{9599369, liu2023efficient}.

However, most existing federated learning frameworks adopt a centralized and synchronous training paradigm, where a central server collects model updates from all clients and aggregates them in each training round. While such designs perform well under ideal network conditions, they face multiple challenges in real-world systems. First, the central server introduces a single point of bottleneck and failure, which limits scalability due to constraints on communication bandwidth, computation capacity and reliability in large-scale deployments \cite{dai2022dispfl}. Second, synchronous training requires strict temporal alignment of clients at round boundaries, making the system vulnerable to device heterogeneity and system variability and leading to the well-known straggler problem that significantly degrades training efficiency \cite{lang2024stragglers, jiang2022towards}. Finally, centralized architectures are often difficult to deploy in cross-organizational or cross-domain scenarios and conflict with the growing demand for trustless and autonomous systems \cite{dai2022dispfl, yuan2024decentralized}.

To overcome these limitations, Decentralized Federated Learning (DFL) has attracted increasing attention. In DFL, clients exchange model information only with their neighboring nodes without relying on a central server, thereby improving system robustness and scalability \cite{9850408, yuan2024decentralized}. Nevertheless, most existing decentralized approaches still rely on synchronous update mechanisms, assuming that clients remain aligned across logical training rounds, an assumption that is difficult to satisfy in dynamic network environments \cite{bornstein2022swift, liu2024aedfl}.

Consequently, Asynchronous Decentralized Federated Learning (ADFL) has been recognized as a more practical training paradigm for real-world deployments \cite{liu2024aedfl, dhasade2025practical}. ADFL allows clients to perform local updates and communications at different times without global synchronization, while supporting model collaboration solely through neighbor-to-neighbor communication \cite{liu2022asynchronous, 11016099}. This setting naturally accommodates device heterogeneity, network variability and partial participation and enables clients to join the training process at arbitrary times \cite{rivero2022data, ameur2022peer}. As a result, ADFL is well suited for edge computing, peer-to-peer networks and large-scale distributed systems, offering significant advantages in terms of system throughput, fault tolerance and deployment flexibility.

Despite these advantages, such a highly flexible training paradigm also introduces new challenges. Specifically, ADFL systems often face three fundamental difficulties simultaneously. First, since clients communicate only with their neighbors and model exchanges may occur frequently, directly transmitting full model parameters incurs prohibitive communication overhead, which is difficult to sustain in bandwidth-constrained or large-scale systems \cite{mcmahan2017communication, 9546506}. Compared to centralized settings, decentralized architectures lack unified communication scheduling and aggregation, making communication efficiency a more prominent concern \cite{yuan2024decentralized, lalitha2018fully}. 
Second, asynchronous updates combined with asymmetric communication topologies can lead to biased aggregation \cite{1238221, assran2019stochastic}. 
In practice, clients may push models at different times and communication links can be unbalanced, under which naive neighbor averaging fails to guarantee unbiased consensus \cite{11016099, franceschelli2009consensus}. 
Finally, the widely observed statistical heterogeneity (non-IID data) across clients is further amplified in asynchronous decentralized environments \cite{liu2024aedfl, liu2022asynchronous, ma2024dynamic}. Local updates computed on non-identically distributed data tend to induce significant model drift, while asynchronous communication may delay or mismatch these updates, exacerbating training instability and degrading final performance \cite{wang2024tackling, liu2024fedasmu}.

These challenges are not independent. 
Communication constraints necessitate compression. Compression can amplify client drift under non-IID data. Asynchrony further delays and skews these updates. Together, these effects exacerbate aggregation bias in decentralized training.
Accordingly, the central research question addressed in this work is:

\emph{How can we jointly design compression, de-biased aggregation, and local-update stabilization for accurate and stable ADFL under asymmetric communication and non-IID data?}

To address this research question, we propose a new ADFL framework, termed PushCen-ADFL (Push-Sum Centroid Asynchronous Decentralized Federated Learning). 
The proposed framework is built around lightweight centroid structures and integrates centroid-based compression with push-sum aggregation, where both aggregation and local optimization are constrained in the same centroid representation space, enabling communication-efficient and stable federated training without relying on a central server or global synchronization.

Our main contributions are summarized as follows:
\begin{itemize}

\item 
We propose PushCen-ADFL, an asynchronous framework that unifies centroid-based communication, push-sum de-biasing and 
buffered neighbor aggregation 
into a single decentralized training pipeline, where both aggregation and local optimization operate in the same centroid representation space.

\item 
We design an average-preserving push-sum aggregation with mass splitting and a deduplicated bounded buffer, mitigating bias from asymmetric information flow and preventing stale messages from dominating aggregation.

\item 
We develop a centroid-aligned proximal regularizer that stabilizes local updates in the same compressed centroid space used for communication. This alignment mitigates non-IID data distributions by contracting local trajectories toward a shared compressed reference.

\item 
We provide a convergence analysis for PushCen-ADFL and validate it on CIFAR-10, CIFAR-100 and Tiny-ImageNet, improving accuracy
by up to 6\% over communication-efficient baselines,
while reducing 
per-push transmitted payload by more than 80\% 
less than full-model communication.
\end{itemize}

\section{Related Work}
\subsection{Asynchronous Decentralized FL}
A representative direction is to design \emph{wait-free} or fully asynchronous communication mechanisms to mitigate stragglers. SWIFT proposes a wait-free decentralized FL algorithm that allows each client to proceed without synchronization barriers, while maintaining standard iteration complexity guarantees \cite{bornstein2022swift}. Complementarily, A$^2$CiD$^2$ studies asynchronous decentralized deep learning from an optimization perspective and introduces a continuous local momentum to accelerate randomized gossip-style communication \cite{nabli2023textbf}. 
Beyond communication primitives, a second line of work explicitly addresses the \emph{staleness} and \emph{heterogeneity} issues that arise in ADFL. Ma \emph{et al.} propose dynamic staleness control for ADFL in decentralized topologies, aiming to balance training efficiency with model quality by regulating the tolerated staleness degree and deriving scheduling/control policies with theoretical support \cite{ma2024dynamic}. For heterogeneous devices, Liao \emph{et al.} develop AsyDFL, which integrates neighbor selection and gradient pushing to reduce communication cost and completion time under non-IID data and system heterogeneity \cite{liao2024asynchronous}. More recently, DSpodFL provides a unified viewpoint by modeling both local updates and pairwise exchanges as sporadic random events, thereby capturing time-varying computation/communication patterns \cite{zehtabi2025decentralized}. In addition, AEDFL proposes an ADFL framework that leverages reinforcement learning for neighbor model selection in heterogeneous, serverless training settings \cite{liu2024aedfl}.
Despite these advances, most existing ADFL methods focus on asynchronous scheduling or staleness control while assuming full-precision communication, and thus fail to jointly address frequent peer-to-peer exchanges, aggregation bias on asymmetric communication, and model drift under non-IID data, limiting their practicality in communication-constrained decentralized systems.

\subsection{Communication-Efficient Federated Learning}
A prominent line of research leverages \emph{knowledge distillation} to reduce the communication payload. FedKD adopts a local teacher-student scheme and transmits only a compact student (or distilled knowledge), substantially lowering uplink/downlink costs while preserving accuracy via distillation \cite{wu2022communication}. Another direction exploits \emph{sparsity} to shrink communicated updates: SpaFL induces structured sparsity with low overhead through trainable thresholds, reducing both communication and computation \cite{kim2024spafl}. In addition, \emph{low-precision} methods cut the bit-width of local training and transmitted updates; Li \emph{et al.} show that low-precision local training can already achieve competitive accuracy in FL \cite{li2024low}. Communication efficiency has also been explored under asynchronous decentralized settings: DivShare exchanges model fragments per interaction to better tolerate stragglers and bandwidth heterogeneity, improving wall-clock convergence in peer-to-peer training \cite{biswas2025boosting}. 
Overall, prior communication-efficient FL methods reduce communication cost but are not designed to handle the coupled challenges of aggregation bias and training instability in fully asynchronous decentralized environments.

\section{Problem Formulation}

We consider an ADFL system with $N$ clients,
denoted by $\mathcal{V}=\{1,2,\dots,N\}$.
Client communications are modeled by a directed graph $\mathcal{G}=(\mathcal{V},\mathcal{E})$ induced by asynchronous gossip,
where $(j \rightarrow i)\in\mathcal{E}$ indicates that client $i$ can receive information from client $j$.
The communication topology can be asymmetric and unbalanced and there is no central server;
clients communicate only with their neighbors.
We denote the in-neighbors and out-neighbors of client $i$ as
$\mathcal{N}_i^-$ and $\mathcal{N}_i^+$, respectively.

Each client $i$ holds a private local dataset $\mathcal{D}_i$ with potentially heterogeneous (non-IID) data distributions,
and maintains a local model copy $w_i$.
Let $f_i(w)$ denote the empirical risk on $\mathcal{D}_i$.
We adopt a \emph{uniform client-level objective} and formulate decentralized learning as a consensus optimization problem:
\[
\min_{\{w_i\}_{i=1}^{N}}\ \frac{1}{N}\sum_{i=1}^{N} f_i(w_i)
\quad \text{s.t.}\quad
w_i = w_j,\ \forall (j\rightarrow i)\in\mathcal{E}.
\]

where the constraints $w_i = w_j$ specify the desired agreement among local models at convergence, while local models may differ during asynchronous training.
The system operates in a fully asynchronous manner without global synchronization.
Clients perform local computation and communication at arbitrary times and received messages may correspond to stale model states due to heterogeneous computation speeds and network delays.
Moreover, clients may become active at arbitrary times, and only active clients generate updates and participate in communication.

Under directed and potentially unbalanced topologies, achieving the above consensus objective requires decentralized aggregation mechanisms that preserve the uniform client-level average.
Accordingly, the problem addressed in this work is to design a communication-efficient ADFL framework that enables stable and robust training under asymmetric communication and non-IID data distributions.

\section{Push-Sum Centroid Asynchronous Decentralized Federated Learning}
\subsection{Overview of PushCen-ADFL}
\label{pushcen_workflow}

\begin{figure*}[t]
    \centering
    \includegraphics[width=\linewidth]{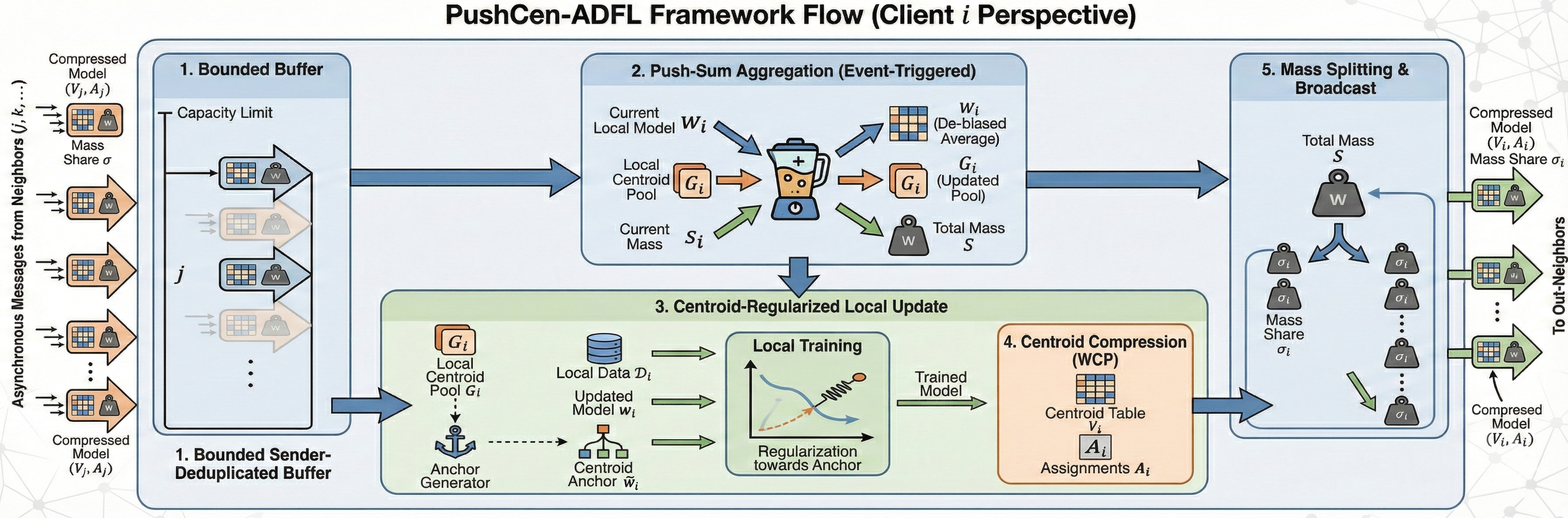}
    \Description{Workflow of PushCen-ADFL: Push-Sum Centroid Asynchronous Decentralized Federated Learning}
    \caption{Workflow of PushCen-ADFL: Push-Sum Centroid Asynchronous Decentralized Federated Learning}
    \label{fig:framework_diagram}
\end{figure*}

The overall workflow of PushCen-ADFL is illustrated in Figure~\ref{fig:framework_diagram} and summarized in Algorithm~\ref{alg:pushcen-main}.

Clients exchange information only with their neighbors over a peer-to-peer communication, without relying on a central coordinator or round-based synchronization.

PushCen-ADFL is built on three coupled principles.
First, each push transmits a centroid message $(V_i, A_i)$ instead of dense parameters, so the communication payload is controlled by the centroid budget.
Second, aggregation follows a mass split push-sum update, using the mass push-sum $s_i$ to correct bias and preserve the averaging under directed asynchronous mixing.
Third, local optimization is stabilized by a centroid proximal term anchored at the current dictionary estimate $G_i$, so that both communication and stabilization operate in the same centroid representation space.
In addition, a deduplicated bounded buffer retains only the most recent message per sender, preventing stale or repeated messages from dominating the event-driven aggregation.

Concretely, each client $i$ maintains a local model $w_i$, a push-sum mass $s_i$ that tracks the mixing weight for de-biased aggregation, and a local centroid dictionary estimate $G_i$ that defines the shared compressed representation space, together with a bounded buffer $\mathcal{B}_i$ that stores recent neighbor messages (line~1).
When a neighbor message $(V_j, A_j, \sigma_{j\rightarrow i}, j)$ arrives, client $i$ updates $\mathcal{B}_i$ via \textsc{BufferUpdate} (lines~3--5), which replaces older entries from the same sender and enforces a fixed buffer capacity; see Section~\ref{sec:buffer}.
Upon a local compute event, client $i$ aggregates buffered information using \textsc{PushSumAgg} (line~6), which jointly updates $(w_i, G_i, s_i)$ in an average-preserving manner under directed and potentially unbalanced communication; see Section~\ref{sec:pushsum_agg}.
Client $i$ then performs \textsc{LocalUpdate} on $\mathcal{D}_i$ (line~7).
For communication, the model is encoded in a centroid form $(V_i, A_i)$ (Section~\ref{sec:centroid_compression}); for optimization, the update is anchored to a centroid-based reference constructed from $G_i$, yielding a centroid-aligned proximal effect that contracts local trajectories and mitigates non-IID drift; see Section~\ref{sec:cen_reg}.
Finally, client $i$ splits its push-sum mass (line~8) and sends $(V_i, A_i, \sigma_i, i)$ to all out-neighbors (line~9), ensuring that subsequent asynchronous aggregations remain average-preserving.

\begin{algorithm}[t]
\caption{PushCen-ADFL: Asynchronous Decentralized Training at Client $i$}
\label{alg:pushcen-main}
\begin{algorithmic}[1]
\Require Out-neighbors $\mathcal{N}_i^+$; local data $\mathcal{D}_i$; \#clusters $K$; local epochs $E$;
learning rate $\eta$; reg weight $\lambda$; total events $T$; buffer limit $L$.
\State \textbf{Init:} $w_i$; $s_i\gets 1$; $G_i\gets\emptyset$; $\mathcal{B}_i\gets\emptyset$
\For{$t=1$ to $T$} \Comment{event-driven; no global synchronization}
    \State \textbf{Upon receiving} $(V_j,A_j,\sigma_{j\rightarrow i},j)$:
    \State \quad \Call{BufferUpdate}{$\mathcal{B}_i,\ (V_j,A_j,\sigma_{j\rightarrow i},j),\ L$}
    \Comment{Alg.~\ref{alg:pushcen-buffer}}
    \State \textbf{Upon a local compute event:}
    \State \quad \Call{PushSumAgg}{$w_i,G_i,s_i,\mathcal{B}_i$}
    \Comment{Alg.~\ref{alg:pushcen-pushsum}}
    \State \quad $(w_i,V_i,A_i)\gets$ \Call{LocalUpdate}{$w_i,\mathcal{D}_i,G_i,K,E,\lambda$}
    \Comment{Alg.~\ref{alg:pushcen-local}}
    \State \quad $\sigma_i \gets s_i/(|\mathcal{N}_i^+|+1)$; $s_i\gets \sigma_i$ \Comment{mass splitting}
    \State \quad Send $(V_i,A_i,\sigma_i,i)$ to all $k\in\mathcal{N}_i^+$
\EndFor
\end{algorithmic}
\end{algorithm}

\begin{algorithm}[t]
\caption{LocalUpdate at Client $i$}
\label{alg:pushcen-local}
\begin{algorithmic}[1]
\Require Model $w_i$, data $\mathcal{D}_i$, centroid dictionary $G_i$, \#clusters $K$, epochs $E$, reg $\lambda$.
\State $(V_i,A_i,M_i)\gets$ \Call{WCP}{$w_i,K,G_i$} \Comment{initialize assignments}
\State Build centroid anchor $\tilde{w}_i \gets \mathrm{Anchor}(G_i,A_i)$
\For{$e=1$ to $E$}
    \State Enforce pruning: $w_i \leftarrow w_i \odot M_i$
    \State Take SGD steps on $f_i(w_i;\mathcal{D}_i)+\lambda\|w_i-\tilde{w}_i\|_2^2$
\EndFor
\State $(V_i,A_i,M_i)\gets$ \Call{WCP}{$w_i,K,G_i$} \Comment{refresh for communication}
\State \Return $(w_i,V_i,A_i)$
\end{algorithmic}
\end{algorithm}

\subsection{Centroid Regularization}
\label{sec:cen_reg}

\begin{figure}
    \centering
    \includegraphics[width=7.5cm]{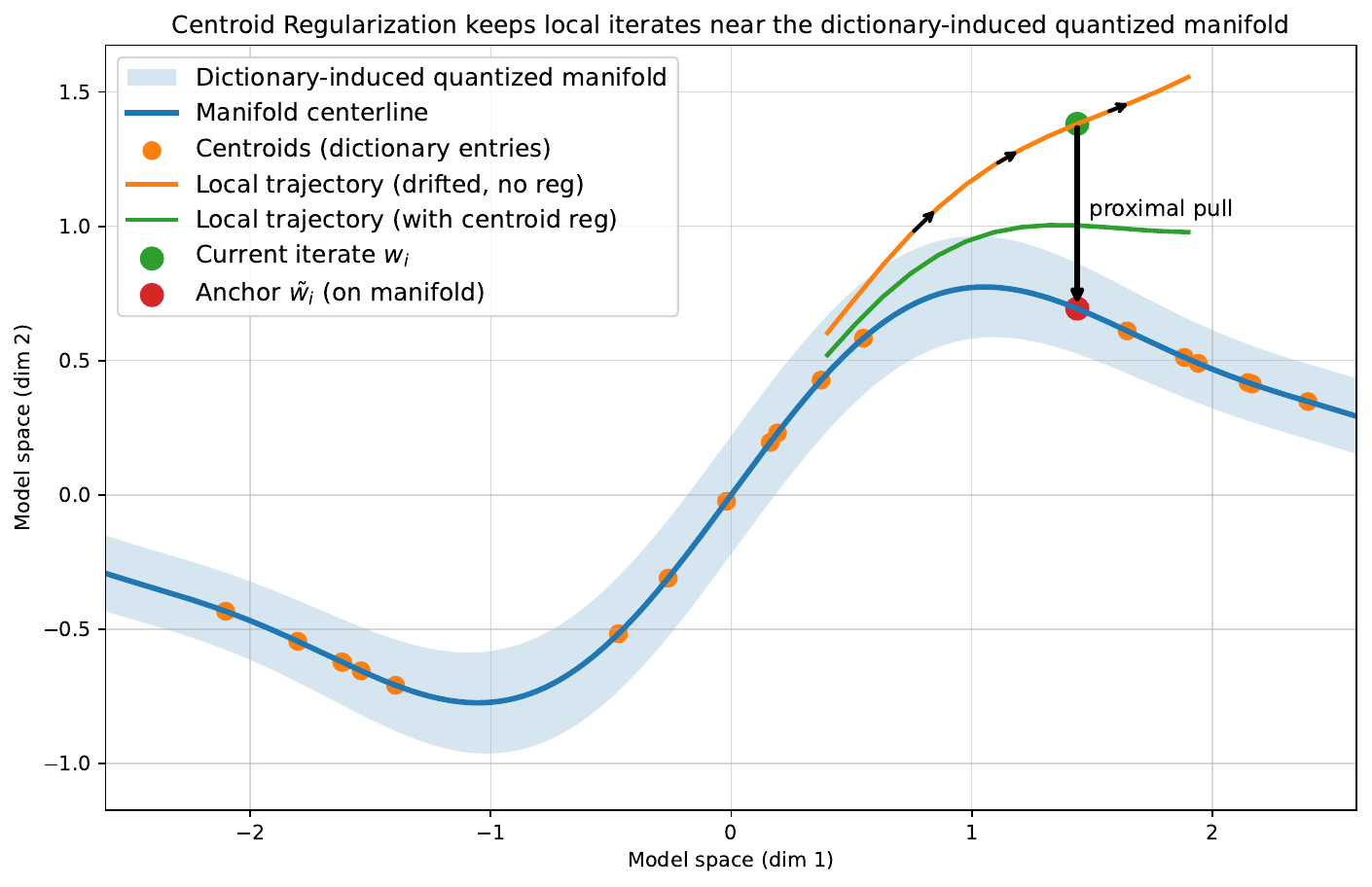}
    \Description{Illustration of Centroid Regularization}
    \caption{Illustration of Centroid Regularization}
    \label{fig:cen_reg_vis}
\end{figure}

Non-IID data can cause client drift, which is further amplified by asynchronous execution. 
To mitigate this drift, PushCen-ADFL introduces a lightweight centroid regularization in the local update step (Algorithm~\ref{alg:pushcen-local}). 
Client $i$ first applies WCP to initialize the assignments $A_i$ and pruning mask $M_i$ (line~1), and then constructs a centroid anchor from its current centroid dictionary estimate $G_i$ (line~2), where $A_i$ specifies how centroids are indexed to form a layer-wise reference.
We instantiate the centroid anchor $\tilde{w}_i$ by decoding the assignment map $A_i$ with the centroid dictionary $G_i$ through an element-wise lookup:
\begin{equation}
\tilde{w}_{i,l} \triangleq G_{i,l}[A_{i,l}],
\label{eq:anchor-def}
\end{equation}
where each entry of $\tilde{w}_{i,l}$ takes the centroid value indexed by the corresponding entry in $A_{i,l}$, yielding the same shape as $w_{i,l}$.

Client $i$ then performs local optimization with a proximal-style regularizer while enforcing pruning:
\begin{equation}
\min_{w_i}\ f_i(w_i;\mathcal{D}_i) + \lambda \|w_i - \tilde{w}_i\|_2^2,
\label{eq:cen-reg}
\end{equation}
where $M_i$ is a binary mask and masked coordinates are kept inactive via $w_i \leftarrow w_i \odot M_i$ (lines~3--5).
Here, $\lambda$ controls the regularization strength and $\tilde{w}_i$ is fixed during the local update.
Unlike conventional proximal terms that pull toward a global model, $\tilde{w}_i$ is decoded from the current centroid dictionary, keeping $w_i$ close to a centroid-structured reference and reducing re-encoding distortion before communication. After local training, client $i$ applies Weight Clustering Pruning (WCP) to refresh the centroid representation $(V_i,A_i)$ for subsequent communication (line~6) and returns $(w_i,V_i,A_i)$ (line~7).

The motivation is to stabilize asynchronous local optimization without decoupling it from centroid-based compression.
In fully asynchronous decentralized training, a client can take many local steps before receiving fresh information; under non-IID data, this can move $w_i$ into regions poorly represented by the current centroid dictionary, leading to unstable trajectories and larger distortion when the model is re-encoded.
The proximal term $\lambda\|w_i-\tilde{w}_i\|_2^2$ provides a local contraction toward $\tilde{w}_i$, curbing drift while keeping the iterate close to the dictionary-induced quantized manifold. Figure~\ref{fig:cen_reg_vis} provides a visualization of centroid regularization.
Because $\tilde{w}_i$ is constructed from $(G_i,A_i)$, the regularizer directly ties the optimization path to the centroid structure used for communication, yielding more stable local updates under asynchrony.

\subsection{Centroid-based Model Compression}
\label{sec:centroid_compression}

\begin{algorithm}[t]
\caption{WCP: Weight Clustering Pruning}
\label{alg:wcp}
\begin{algorithmic}[1]
\Require Model $w_i$; \#clusters $K$; optional init centroids $G_i$; max iter $T_{\max}$.
\State $V_i\gets\emptyset;\ A_i\gets\emptyset;\ M_i\gets\mathbf{1}$
\For{each weight layer $l$ in $w_i$}
    \State $\theta \gets \mathrm{vec}(w_{i,l})$
    \If{$G_{i,l}$ available}
        \State $\mu \gets [\,0;\ G_{i,l}[1{:}K\!-\!1]\,]$
    \Else
        \State $\mu \gets [\,0;\ \mathrm{Rand}(\theta)[1{:}K\!-\!1]\,]$
    \EndIf
    \For{$t=1$ to $T_{\max}$ \textbf{or} converged}
        \State $a \gets \arg\min_{j\in[K]}\|\theta-\mu_j\|_2$ \Comment{element-wise}
        \For{$j=1$ to $K-1$} \State $\mu_j \gets \mathrm{mean}(\theta[a=j])$ \EndFor
        \State $\mu_0 \gets 0$
    \EndFor
    \State $(V_{i,l},A_{i,l}) \gets \mathrm{SortRemap}(\mu,a)$
    \State $w_{i,l} \gets V_{i,l}[A_{i,l}]$;\ \ $M_{i,l}\gets \mathbb{I}(w_{i,l}\neq 0)$
\EndFor
\State \Return $(V_i,A_i,M_i)$
\end{algorithmic}
\end{algorithm}

PushCen-ADFL compresses communication by encoding each weight layer in a centroid form rather than transmitting dense parameters.
For layer $l$, client $i$ vectorizes weights $\theta=\mathrm{vec}(w_{i,l})\in\mathbb{R}^{d_l}$ and clusters its entries into $K$ scalar centroids, producing a centroid table $V_{i,l}\in\mathbb{R}^{K}$ and an assignment map $A_{i,l}\in\{1,\dots,K\}^{d_l}$.
Reconstruction uses an element-wise lookup $w_{i,l}\approx V_{i,l}[A_{i,l}]$, so the payload scales with $K$ (typically $K\ll d_l$) rather than $d_l$.

Algorithm~\ref{alg:wcp} implements this encoding via WCP.
When available, WCP warm-starts centroids from the local dictionary estimate $G_{i,l}\in\mathbb{R}^{K}$ by setting $\mu=[0;\,G_{i,l}[1{:}K\!-\!1]]$, fixing one centroid at $0$ and refining the others by alternating nearest-centroid assignment and cluster-mean updates.
It returns $(V_{i,l},A_{i,l})$ and a binary pruning mask $M_{i,l}$, where \textsc{SortRemap} reorders centroid indices into a canonical order and remaps $A_{i,l}$ accordingly to ensure consistent encoding.
Full cost details are deferred to Appendix~\ref{app:wcp_details}.

\subsection{Asynchronous Push-Sum Aggregation}
\label{sec:pushsum_agg}

\begin{algorithm}[t]
\caption{PushSumAgg}
\label{alg:pushcen-pushsum}
\begin{algorithmic}[1]
\Require Local $(w_i,G_i,s_i)$; buffer $\mathcal{B}_i=\{(V_j,A_j,\sigma_{j\rightarrow i},j)\}$.
\If{$\mathcal{B}_i=\emptyset$} \State \Return \EndIf
\State For each buffered message, let $\hat{w}_j \gets \mathcal{R}(V_j,A_j)$ \Comment{implicit reconstruction}
\State Compute total mass $S$ by Eq.~\eqref{eq:pushsum-mass}
\State Update model $w_i$ by Eq.~\eqref{eq:pushsum-model} using $\{\hat{w}_j\}$
\State Update centroid dictionary $G_i$ by Eq.~\eqref{eq:pushsum-pool} using $\{V_j\}$
\State $s_i \gets S$; clear $\mathcal{B}_i$
\end{algorithmic}
\end{algorithm}
Asynchronous decentralized training is inherently imbalanced: fast clients (due to higher compute speed or lower latency) can generate and disseminate updates more frequently, which may cause their local trajectories to dominate aggregation when using naive neighbor averaging. 
This issue is compounded on directed or unbalanced graphs, where information flow is asymmetric.
PushCen-ADFL addresses both effects with an average-preserving push-sum style aggregation that tracks a scalar mass and uses it to normalize heterogeneous message rates and topology asymmetry.

Each client $i$ maintains a local model $w_i$, a push-sum mass $s_i>0$, a local estimate of the centroid dictionary $G_i$ and a buffer $\mathcal{B}_i$ that stores neighbor messages received.
A buffered message from neighbor $j$ is denoted by $(V_j, A_j, \sigma_{j\rightarrow i}, j)$, where $(V_j,A_j)$ is the centroid representation and $\sigma_{j\rightarrow i}$ is the mass share sent to $i$.
Upon aggregation, client $i$ implicitly reconstructs a quantized neighbor model $\hat{w}_j=\mathcal{R}(V_j,A_j)$ (Algorithm~\ref{alg:pushcen-pushsum}, line~4).
It then computes the total mass
\begin{align}
S \leftarrow s_i + \sum_{(V_j,A_j,\sigma_{j\rightarrow i},j)\in\mathcal{B}_i}\sigma_{j\rightarrow i},
\label{eq:pushsum-mass}
\end{align}
which serves as the normalization constant for mass-conserving averaging.
Client $i$ updates its model via
\begin{align}
w_i \leftarrow \frac{s_i}{S}w_i + \sum_{(V_j,A_j,\sigma_{j\rightarrow i},j)\in\mathcal{B}_i}\frac{\sigma_{j\rightarrow i}}{S}\,\hat{w}_j,
\label{eq:pushsum-model}
\end{align}

so that a sender’s influence is determined by conserved mass rather than raw message frequency.
The same weights are applied to update the centroid dictionary estimate using the received centroid tables.
\begin{align}
G_i \leftarrow \frac{s_i}{S}G_i + \sum_{(V_j,A_j,\sigma_{j\rightarrow i},j)\in\mathcal{B}_i}\frac{\sigma_{j\rightarrow i}}{S}\,V_j.
\label{eq:pushsum-pool}
\end{align}
aligning the shared centroid structure with the aggregated model trajectory.
After aggregation, client $i$ sets $s_i\leftarrow S$ and clears $\mathcal{B}_i$.

After local training, client $i$ performs mass splitting before communication.
Let $d_i^{+}=|\mathcal{N}_i^{+}|$ be the out-degree; client $i$ sends $\sigma_i=s_i/(d_i^{+}+1)$ to each out-neighbor and updates $s_i\leftarrow \sigma_i$.
This mass-splitting rule prevents fast clients from dominating the network dynamics under asynchrony.

\subsection{Buffered Neighbor Updates}
\label{sec:buffer}

\begin{algorithm}[t]
\caption{BufferUpdate (Deduplicate by Sender + Bounded FIFO)}
\label{alg:pushcen-buffer}
\begin{algorithmic}[1]
\Require Buffer $\mathcal{B}_i$; incoming $(w_j,V_j,\sigma_{j\rightarrow i},j)$; limit $L$.
\State Remove the existing entry from sender $j$ in $\mathcal{B}_i$ (if any)
\State Append $(w_j,V_j,\sigma_{j\rightarrow i},j)$ to $\mathcal{B}_i$
\If{$L>0$ and $|\mathcal{B}_i|>L$}
    \State Drop the oldest $|\mathcal{B}_i|-L$ entries from $\mathcal{B}_i$
\EndIf
\end{algorithmic}
\end{algorithm}

In a fully asynchronous decentralized setting, message arrivals are decoupled from local compute events. A client may receive neighbor messages at arbitrary times, possibly in bursts, and some messages can be stale by the time aggregation is triggered.
PushCen-ADFL therefore maintains a local buffer $\mathcal{B}_i$ at each client $i$ to cache received neighbor information until the next compute event, when aggregation is performed.

The buffer is designed to be sender-deduplicated and capacity-bounded.
When a new message from neighbor $j$ arrives, client $i$ removes any existing buffered entry from $j$ and keeps only the latest one.
This enforces the principle that each sender contributes at most one message per aggregation event, preventing bursty arrivals from being counted multiple times and distorting the intended neighbor weighting, which per-message push-sum normalization alone does not prevent.

Empirically, removing sender deduplication consistently degrades performance (Section~\ref{sec:ablation_study}), indicating that discarding older messages from the same neighbor is beneficial in practice.
Finally, to control memory and limit the influence of long-delayed messages under high communication rates, we impose a buffer size limit and drop the oldest entries upon overflow.
Together, these rules provide a lightweight interface between asynchronous communication and push-sum aggregation, as summarized in Algorithm~\ref{alg:pushcen-buffer}.

\subsection{Convergence Analysis of PushCen-ADFL}
\label{sec:convergence}
\subsubsection{Assumptions}
We summarize the assumptions used in our convergence analysis; full definitions, update rules and proofs are deferred to Appendix~\ref{app:conv}.

\paragraph{Assumption 1: Smoothness and lower boundedness \cite{liconvergence, karimireddy2020scaffold}.}
Each $f_i$ is $L$-smooth and $F(w)\triangleq \frac{1}{N}\sum_{i=1}^N f_i(w)$ is lower bounded by $F^\star$.
\paragraph{Assumption 2: Stochastic gradients \cite{liconvergence, karimireddy2020scaffold}.} $\mathbb{E}[g_i(w;\xi)\mid w]=\nabla f_i(w)$ and
$\mathbb{E}[\|g_i(w;\xi)-\nabla f_i(w)\|_2^2\mid w]\le \sigma^2$.
\paragraph{Assumption 3: Uniform heterogeneity \cite{karimireddy2020scaffold}.}
For all $i$ and $w$, $\|\nabla f_i(w)-\nabla F(w)\|_2^2\le \zeta^2$.
\paragraph{Assumption 4: Bounded staleness \cite{nguyen2022federated, assran2019stochastic}.}
Any message used at event $t$ is generated within the last $\tau$ events.
\paragraph{Assumption 5: Directed push-sum mixing \cite{nedic2014distributed}.}
The time-varying directed communication is $B$-window strongly connected and push-sum masses stay positive.
\paragraph{Assumptions 6: Absolute compression error \cite{danilova2022distributed}.} $\|\mathcal{R}(\mathcal{C}(w)) - w\|_2 \le \varepsilon_c$ for any transmitted model $w$.

To simplify constants, we further assume bounded push-sum masses, bounded iterates, and bounded gradient second moments. These are given in Assumptions~\ref{ass:mass-bounded} in Appendix~\ref{app:conv}.

\subsubsection{De-biased reference and stationarity metric}
On directed and potentially unbalanced graphs, the arithmetic mean is not preserved. We therefore analyze the de-biased
push-sum reference $\bar w^t \triangleq X_{\mathrm{tot}}^t/Y_{\mathrm{tot}}^t$ (defined via system-level accounting that
includes buffered/in-flight messages; Appendix~\ref{app:conv}) and measure convergence by
$\frac{1}{T}\sum_{t=0}^{T-1}\mathbb{E}\|\nabla F(\bar w^t)\|_2^2$.

\subsubsection{Main theorem}
\begin{theorem}[Stationarity of PushCen-ADFL]
\label{thm:main}
Suppose Assumptions1 - 6 and Assumptions~\ref{ass:mass-bounded} hold.
Let the stepsize satisfy $\eta \le \min\Big\{\frac{1}{8LE},\ \frac{1}{4\lambda}\Big\}$.
Then
\begin{equation}
\label{eq:main-body}
\begin{aligned}
\frac{1}{T}\sum_{t=0}^{T-1}\mathbb{E}\big[\|\nabla F(\bar w^t)\|_2^2\big]
\;&\le\;
\frac{2\big(F(\bar w^0)-F^\star\big)}{\eta E\,\underline{\gamma}\,T}
+ 2(\sigma^2+\zeta^2)
+ 2L\,\overline{c}^{\,2}\varepsilon_c^2
\\
&\quad+\;
\frac{4L^2 N}{\underline{\gamma}}
\left(
\frac{\mathbb{E}[\mathcal{E}_{\mathrm{con}}^0]}{T(1-\rho)}
+
\frac{B_{\mathrm{con}}}{1-\rho}
\right)
+
\frac{L}{\underline{\gamma}}\,D_\Delta .
\end{aligned}
\end{equation}

where $\underline{\gamma}=\frac{y_{\min}}{Ny_{\max}}$, $\rho\in(0,1)$ is the contraction factor in the directed push-sum
consensus recursion (Lemma~\ref{lem:consensus-recursion} in Appendix~\ref{app:conv}) and
\begin{equation}
\overline{c}
\triangleq
\max_t \frac{d_{i_t}^+}{d_{i_t}^+ + 1}\cdot \frac{y_{i_t}^{t+\frac13}}{Y_{\mathrm{tot}}}
\le
\frac{y_{\max}}{Ny_{\min}}.
\label{eq:cbar-body}
\end{equation}
Moreover,
\begin{align}
D_\Delta
&\triangleq
2E^2\eta^2\Big(G^2 + 16\lambda^2W^2\Big),
\label{eq:Ddelta-body}
\\
B_{\mathrm{con}}
&\triangleq
\big(C_1 + C_2\tau^2\big)D_\Delta + C_3\varepsilon_c^2,
\label{eq:Bcon-body}
\end{align}
with $(C_1,C_2,C_3)$ from Lemma~\ref{lem:consensus-recursion}.
\end{theorem}

\section{Experiments}

\begin{table*}[t]
\centering
\rowcolors{10}{white}{gray!30}
\caption{Test accuracy (\%) and standard deviation (SD) under different data heterogeneity levels $\alpha$.}
\label{tab:accuracy}
\setlength{\tabcolsep}{4pt}
\small
\begin{tabular}{l 
cc cc cc | 
cc cc cc | 
cc cc cc}
\toprule
\multirow{3}{*}{Method}
& \multicolumn{6}{c|}{CIFAR-10}
& \multicolumn{6}{c|}{CIFAR-100}
& \multicolumn{6}{c}{Tiny-ImageNet} \\
\cmidrule(lr){2-7}\cmidrule(lr){8-13}\cmidrule(lr){14-19}
& \multicolumn{2}{c}{$\alpha=0.1$}
& \multicolumn{2}{c}{$\alpha=0.4$}
& \multicolumn{2}{c|}{$\alpha=1.0$}
& \multicolumn{2}{c}{$\alpha=0.1$}
& \multicolumn{2}{c}{$\alpha=0.4$}
& \multicolumn{2}{c|}{$\alpha=1.0$}
& \multicolumn{2}{c}{$\alpha=0.1$}
& \multicolumn{2}{c}{$\alpha=0.4$}
& \multicolumn{2}{c}{$\alpha=1.0$} \\
\cmidrule(lr){2-3}\cmidrule(lr){4-5}\cmidrule(lr){6-7}
\cmidrule(lr){8-9}\cmidrule(lr){10-11}\cmidrule(lr){12-13}
\cmidrule(lr){14-15}\cmidrule(lr){16-17}\cmidrule(lr){18-19}
& Acc.\,$\uparrow$ & SD\,$\downarrow$ & Acc.\,$\uparrow$ & SD\,$\downarrow$ & Acc.\,$\uparrow$ & SD\,$\downarrow$
& Acc.\,$\uparrow$ & SD\,$\downarrow$ & Acc.\,$\uparrow$ & SD\,$\downarrow$ & Acc.\,$\uparrow$ & SD\,$\downarrow$
& Acc.\,$\uparrow$ & SD\,$\downarrow$ & Acc.\,$\uparrow$ & SD\,$\downarrow$ & Acc.\,$\uparrow$ & SD\,$\downarrow$ \\
\midrule
Independent
& 43.85 & 6.55 & 54.06 & 8.94 & 68.01 & 10.80 & 22.79 & 3.64 & 33.09 & 5.06 & 52.86 & 6.17 & 21.69 & 3.01 & 32.37 & 3.50 & 53.56 & 4.38 \\
\midrule
\multicolumn{19}{l}{\small\textit{Full-communication}} \\
Async-DFedAvg
& 57.72 & 6.84 & 63.05 & 8.70 & 74.22 & 9.55 & 41.14 & 4.28 & 46.60 & 3.75 & 60.06 & 7.54 & 43.79 & 2.68 & 50.97 & 3.00 & 63.56 & 3.99 \\
SWIFT
& 56.93 & 6.46 & 64.11 & 8.11 & 72.59 & 9.69 & 41.92 & 4.36 & 47.59 & 4.07 & 50.96 & 7.75 & 43.68 & 2.68 & 50.65 & 3.33 & 63.75 & 3.85 \\
\midrule
\multicolumn{19}{l}{\small\textit{Communication-efficient}} \\
DivShare
& 51.96 & 6.62 & 60.25 & 8.53 & 69.46 & 9.88 & 32.93 & 3.44 & 40.64 & 4.39 & 56.67 & 6.52 & 30.52 & 3.03 & 38.75 & 3.49 & 56.05 & 3.99 \\
\textbf{PushCen-ADFL}
& \textbf{58.30} & \textbf{6.32} & \textbf{66.22} & \textbf{7.82} & \textbf{73.89} & \textbf{9.28} & \textbf{35.15} & \textbf{3.59} & \textbf{44.21} & \textbf{4.25} & \textbf{57.23} & \textbf{7.58} & \textbf{36.70} & \textbf{2.69} & \textbf{44.62} & \textbf{3.43} & \textbf{58.04} & \textbf{3.99} \\
\bottomrule
\end{tabular}
\end{table*}

\begin{table*}[t]
\centering
\rowcolors{7}{gray!30}{white}
\caption{Per-push communication cost and relative overhead (normalized to PushCen-ADFL).}
\label{tab:comm-cost-relative}
\setlength{\tabcolsep}{5pt}
\small
\sisetup{detect-weight=true, table-number-alignment=center}
\begin{tabular}{l
S[table-format=3.3] S[table-format=1.3]
S[table-format=3.3] S[table-format=1.3]
S[table-format=3.3] S[table-format=1.3]}
\toprule
\multirow{2}{*}{Method}
& \multicolumn{2}{c}{CIFAR-10}
& \multicolumn{2}{c}{CIFAR-100}
& \multicolumn{2}{c}{Tiny-ImageNet} \\
\cmidrule(lr){2-3}\cmidrule(lr){4-5}\cmidrule(lr){6-7}
& \multicolumn{1}{c}{Comm. (MB)\,$\downarrow$} & \multicolumn{1}{c}{$\times$ Async\,$\downarrow$}
& \multicolumn{1}{c}{Comm. (MB)\,$\downarrow$} & \multicolumn{1}{c}{$\times$ Async\,$\downarrow$}
& \multicolumn{1}{c}{Comm. (MB)\,$\downarrow$} & \multicolumn{1}{c}{$\times$ Async\,$\downarrow$} \\
\midrule
Async-DFedAvg  & 10.27 & 6.33 & 449.11 & 5.88 & 451.16 & 5.89 \\
DivShare       &  2.06 & 1.27 &  89.82 & 1.18 &  90.23 & 1.18 \\
SWIFT          & 10.27 & 6.33 & 449.11 & 5.88 & 451.16 & 5.89 \\
\textbf{PushCen-ADFL} & \textbf{1.622} & \textbf{1.00} & \textbf{76.33} & \textbf{1.00} & \textbf{76.65} & \textbf{1.00} \\
\bottomrule
\end{tabular}
\end{table*}

\subsection{Experimental Setup}
\subsubsection{Experimental Environment}
We conduct experiments on a Slurm-managed cluster with two compute nodes. Node A has an Intel Core i9-13900K CPU, 64GB RAM and one NVIDIA GeForce RTX 4090 GPU; Node B has an Intel Xeon Silver 4309Y CPU, 128GB RAM and two NVIDIA A40 GPUs. All methods are implemented in Python. Our method and all baselines are implemented in PyTorch \footnote{https://pytorch.org} and trained/inferred on GPUs. To emulate client asynchrony and message arrivals in asynchronous decentralized FL, we implement an event-driven simulator that schedules 
local computation and communication events independently for each client, allowing updates and message deliveries to occur at arbitrary times.
Unless otherwise stated, all experiments use a fixed random seed, set $K=32$, and cap the FIFO buffer at $L=16$ to bound memory under bursty asynchronous arrivals. Each client pushes to 10 random neighbors via a gossip protocol.
The code was released at \footnote{https://github.com/SHVleV9CYWkK/ADFLab}.

\subsubsection{Datasets and Models}
We evaluate our method on three standard vision classification benchmarks: CIFAR-10, CIFAR-100 and Tiny-ImageNet. Following common practice, we pair datasets with architectures of appropriate capacity: LeNet for CIFAR-10 and ResNet-18 for CIFAR-100 and Tiny-ImageNet \cite{krizhevsky2009learning, le2015tiny, lecun1998gradient, he2016deep}. All methods are trained under the same data preprocessing and training pipeline to ensure fair comparisons.

\subsubsection{Data Partitioning and Heterogeneity}
To emulate the statistical heterogeneity (non-IID data) commonly observed in real-world federated learning, we partition each dataset using a Dirichlet distribution with concentration parameters $\alpha\in\{0.1,0.4,1.0\}$, where smaller $\alpha$ induces more skewed client data distributions. CIFAR-10 is partitioned across 100 clients, while CIFAR-100 and Tiny-ImageNet are each distributed among 50 clients
to avoid overly small per-client datasets under Dirichlet splits, which can make training unstable.
For evaluation under each client’s local distribution, we further split each client’s data into local training and test sets and report performance on the corresponding local test data.

\subsubsection{Delayed Client Scenario and Evaluation Protocol}
To evaluate delayed client participation in asynchronous decentralized federated learning, we adopt a round-free 
(pseudo-time–based)
evaluation protocol consistent with ADFL. Training proceeds over a fixed global pseudo-time horizon, uniformly divided into 60 evaluation intervals, at which the average Top-1 accuracy of all online clients is reported. For each dataset, 10\% of clients are randomly designated as delayed clients prior to training, while all remaining clients participate from the beginning. Each delayed client joins the system at a randomly sampled time during training. To ensure a fair comparison, both the identities of delayed clients and their join times are fixed and shared across all compared methods.

\subsubsection{Baseline Methods}
We compare against four representative baselines. \textbf{Async-DFedAvg} serves as a canonical benchmark, capturing a standard asynchronous decentralized adaptation of the FedAvg-style aggregation paradigm \cite{mcmahan2017communication}. \textbf{DivShare} represents a state-of-the-art communication-efficient approach in asynchronous decentralized settings and is included to benchmark communication efficiency \cite{biswas2025boosting}. \textbf{Independent} performs purely local training without any model exchange, providing a non-collaborative lower bound. Finally, \textbf{SWIFT} is a representative asynchronous decentralized FL algorithm and is commonly used as a baseline for handling statistical heterogeneity under asynchrony \cite{bornstein2022swift}.

\subsection{Experimental Results}
\subsubsection{Model Accuracy}
Table~\ref{tab:accuracy} reports the Top-1 accuracy together with the corresponding standard deviation (SD) under different levels of non-IID data.
In addition to mean performance, the SD provides a sensitivity analysis that characterizes the variability of model accuracy across clients and training trajectories in asynchronous decentralized settings.
On CIFAR-10, PushCen-ADFL achieves the best accuracy under high heterogeneity, with particularly strong gains as data become more non-IID.
At $\alpha=0.4$ and $\alpha=0.1$, PushCen-ADFL reaches 66.22\% and 58.30\%, respectively, outperforming the communication-efficient baseline DivShare by clear margins, while exhibiting SDs that are comparable to and often lower than, those of competing methods on CIFAR-10.
When heterogeneity is mild ($\alpha=1.0$), PushCen-ADFL remains highly competitive with full-communication baselines and shows similar variability.
On more challenging datasets (CIFAR-100 and Tiny-ImageNet), PushCen-ADFL consistently outperforms DivShare across all heterogeneity levels, with more pronounced improvements under stronger non-IID distributions.
For example, at $\alpha=0.1$, PushCen-ADFL improves accuracy by 2.22\% on CIFAR-100 and 6.18\% on Tiny-ImageNet, while maintaining overall variability comparable to the baselines across heterogeneity levels.
PushCen-ADFL achieves competitive mean accuracy and comparable SDs relative to full-communication methods, demonstrating that its centroid-based compression and regularization enable strong, robust performance in heterogeneous and asynchronous decentralized training without increasing sensitivity to updates.

\subsubsection{Communication Cost}
Table~\ref{tab:comm-cost-relative} reports the communication cost of a single client push event, which is the total payload transmitted when a client broadcasts its message to 10 neighbors. This metric isolates the \emph{per-push} communication footprint from the overall pseudo-time horizon and thus directly reflects the bandwidth burden induced by one communication action in our asynchronous decentralized protocol. 
We focus on the per-push cost because total system communication scales with the per-push payload times the number of push events. The event count depends on asynchrony, client delays, and scheduling, rather than the message format itself.
As expected, full-precision methods that transmit dense model information (Async-DFedAvg and SWIFT) incur substantially larger per-push traffic, reaching 10.27\,MB on CIFAR-10 and about 449 to 451\,MB on CIFAR-100 and Tiny-ImageNet. In contrast, communication-efficient baselines significantly reduce the payload size. DivShare lowers the per-push cost to 2.06\,MB (CIFAR-10) and 89.82 to 90.23\,MB (CIFAR-100 and Tiny-ImageNet), while PushCen-ADFL achieves the smallest per-push footprint across all datasets, requiring only 1.62\,MB on CIFAR-10 and 76.33 to 76.65\,MB on CIFAR-100 and Tiny-ImageNet. 

We further normalize the per-push communication volume to PushCen-ADFL, where PushCen-ADFL is 1.00 by definition. Under this normalization, DivShare requires 1.27$\times$ more bandwidth per push on CIFAR-10 and 1.18$\times$ on CIFAR-100 and Tiny-ImageNet, whereas Async-DFedAvg and SWIFT require 6.33$\times$ on CIFAR-10 and about 5.88 to 5.89$\times$ on CIFAR-100 and Tiny-ImageNet. Overall, these results demonstrate that PushCen-ADFL substantially reduces the communication footprint of each decentralized push, which is particularly important in asynchronous settings where communication are frequent and uncoordinated.

\subsubsection{Computational overhead}
\begin{table}[t]
\centering
\rowcolors{2}{gray!30}{white}
\caption{Computational cost (GFLOPs) of different methods on different datasets.}
\label{tab:gflops}
\setlength{\tabcolsep}{8pt}
\small
\begin{tabular}{l c c c}
\toprule
Method & CIFAR-10 & CIFAR-100 & Tiny-ImageNet \\
\midrule
Async-DFedAvg  & 4.15 & 10.70 & 42.44 \\
DivShare       & 4.15 & 10.61 & 42.35 \\
SWIFT          & 4.16 & 10.82 & 42.56 \\
\textbf{PushCen-ADFL}  & \textbf{4.03} & \textbf{11.93} & \textbf{42.41} \\
\bottomrule
\end{tabular}
\end{table}

Table~\ref{tab:gflops} reports the computational cost (GFLOPs) of different methods on three datasets. Overall, PushCen-ADFL maintains a lightweight computation footprint comparable to baselines. On CIFAR-10, PushCen-ADFL slightly reduces the cost to 4.03 GFLOPs, marginally lower than Async-DFedAvg, DivShare and SWIFT. On Tiny-ImageNet, it achieves 42.41 GFLOPs, essentially matching the baseline range. On CIFAR-100, PushCen-ADFL reaches 11.93 GFLOPs versus 10.61--10.82 GFLOPs for other methods, remaining within a manageable range.

We observe a slightly higher computation cost on CIFAR-100, mainly due to WCP-related operations (centroid clustering and re-encoding) and the centroid-based anchor construction during local updates. This effect is more visible with $K=32$, which yields less aggressive sparsity and thus provides limited training-time computation reduction while still incurring clustering overhead. 

\subsubsection{Ablation Study}
\label{sec:ablation_study}
\begin{table}[t]
\rowcolors{2}{white}{gray!30}
\centering
\caption{Ablation study on different datasets.}
\label{tab:ablation}
\setlength{\tabcolsep}{4pt}
\small
\begin{tabular}{lccc}
\toprule
Method & CIFAR-10 & CIFAR-100  & Tiny-ImageNet \\
\midrule
PushCen-ADFL (\emph{No Reg.})  & 65.30\% & 42.82\% & 43.93\% \\
PushCen-ADFL (\emph{No Buffer.}) & 65.04\% & 32.01\% & 43.85\% \\
\textbf{PushCen-ADFL}    & \textbf{66.22\%} & \textbf{44.21\%} & \textbf{44.62\%} \\
\bottomrule
\end{tabular}
\end{table}

Table~\ref{tab:ablation} evaluates the contribution of two key components in PushCen-ADFL: the regularization term and the buffer-based message handling mechanism. Removing the regularization (\emph{No Reg}) leads to a consistent decrease in accuracy across all datasets, from 66.22\% to 65.30\% on CIFAR-10, from 44.21\% to 42.82\% on CIFAR-100 and from 44.62\% to 43.93\% on Tiny-ImageNet. This indicates that the regularization is beneficial for stabilizing training and improving generalization under heterogeneous and asynchronous decentralized updates. Disabling the buffer mechanism (\emph{No Buffer}) also degrades performance, with the most pronounced effect observed on CIFAR-100 (44.21\% to 32.01\%), suggesting that buffer-based handling of incoming messages plays an important role in mitigating adverse effects caused by asynchronous arrivals (e.g., redundant or highly stale information) on more challenging tasks. Overall, both components contribute to the final performance and the full PushCen-ADFL achieves the best results across all datasets.

\subsection{Delayed Client Experiment Results}
\begin{table}[t]
\centering
\rowcolors{2}{white}{gray!30}
\caption{Delayed-client performance in terms of the mean maximum test accuracy (\%) and standard deviation (SD).}
\label{tab:delayed_clients}
\setlength{\tabcolsep}{5pt}
\small
\begin{tabular}{l cc cc cc}
\toprule
\multirow{2}{*}{Method}
& \multicolumn{2}{c}{CIFAR-10}
& \multicolumn{2}{c}{CIFAR-100}
& \multicolumn{2}{c}{Tiny-ImageNet} \\
\cmidrule(lr){2-3}\cmidrule(lr){4-5}\cmidrule(lr){6-7}
& Acc.\,$\uparrow$ & SD\,$\downarrow$
& Acc.\,$\uparrow$ & SD\,$\downarrow$
& Acc.\,$\uparrow$ & SD\,$\downarrow$ \\
\midrule
Independent
& 60.99 & 0.944
& 33.94 & 0.301
& 31.06 & 0.037 \\
\midrule
\multicolumn{7}{l}{\small\textit{Full-communication}} \\
Async-DFedAvg
& 68.54 & 0.436
& 49.43 & 0.289
& 50.32 & 0.029 \\
Swift
& 68.65 & 0.502
& 50.37 & 0.337
& 49.66 & 0.059 \\
\midrule
\multicolumn{7}{l}{\small\textit{Communication-efficient Methods}} \\
DivShare
& 67.47 & 0.474
& 44.48 & 0.182
& 37.71 & 0.080 \\
\textbf{PushCen-ADFL}
& \textbf{72.77} & \textbf{0.356}
& \textbf{46.50} & \textbf{0.321}
& \textbf{43.08} & \textbf{0.034} \\
\bottomrule
\end{tabular}
\end{table}

Table~\ref{tab:delayed_clients} reports delayed-client performance measured by the mean of each delayed client’s maximum test accuracy during training, together with the corresponding standard deviation (SD). 
On CIFAR-10, PushCen-ADFL reaches an accuracy of 72.77\%, exceeding both the full-communication baselines and the communication-efficient. 
On CIFAR-100 and Tiny-ImageNet, PushCen-ADFL attains 46.50\% and 43.08\%, respectively, consistently improving over DivShare while remaining competitive with full-communication methods. 
These results indicate that PushCen-ADFL is particularly effective at incorporating delayed clients and enabling them to achieve strong personalized performance under their local distributions.
Beyond the mean accuracy, SD shows that PushCen-ADFL also yields low variability across delayed clients, achieving the smallest SD on CIFAR-10 (0.356) and comparable to full-communication baselines on CIFAR-100 and Tiny-ImageNet.
Detailed training curves for delayed clients are provided in Appendix~\ref{app:delayed_client_learn_curves}.

\subsection{Hyperparameter Sensitivity}
\label{sec:hyperparam}

\subsubsection{Effect of centroid number $K$.}
Table~\ref{tab:k-analysis} shows that increasing $K$ improves accuracy at the cost of higher
communication and computation.
Even with $K = 32$, PushCen-ADFL keeps per-push cost and GFLOPs strictly below all baselines,
confirming that $K$ provides a practical knob to trade accuracy for efficiency.
We use $K=32$ by default as it yields the best accuracy across all datasets.

\begin{table}[t]
\centering
\rowcolors{2}{white}{gray!30}
\caption{Effect of centroid number $K$ on accuracy and per-push communication cost.}
\label{tab:k-analysis}
\setlength{\tabcolsep}{2pt}
\small
\begin{tabular}{c cc cc cc}
\toprule
\multirow{2}{*}{$K$}
& \multicolumn{2}{c}{CIFAR-10}
& \multicolumn{2}{c}{CIFAR-100}
& \multicolumn{2}{c}{Tiny-ImageNet} \\
\cmidrule(lr){2-3}\cmidrule(lr){4-5}\cmidrule(lr){6-7}
& Acc.\,(\%) & Comm.\,(MB)
& Acc.\,(\%) & Comm.\,(MB)
& Acc.\,(\%) & Comm.\,(MB) \\
\midrule
8  & 52.45 & 0.976 & 25.51 & 48.70 & 15.83 & 48.89 \\
16 & 62.17 & 1.298 & 36.67 & 62.51 & 35.96 & 62.77 \\
\textbf{32} & \textbf{66.22} & \textbf{1.622} & \textbf{44.21} & \textbf{76.33} & \textbf{44.62} & \textbf{76.65} \\
\bottomrule
\end{tabular}
\end{table}

\subsubsection{Effect of regularization weight $\lambda$.}
Table~\ref{tab:lambda} shows that enabling centroid regularization ($\lambda>0$) consistently
outperforms $\lambda=0$ across all datasets, confirming its role in stabilizing local updates
under heterogeneity and asynchrony.
On CIFAR-10, performance is stable for $\lambda\in[0.01,0.10]$ (${\approx}66\%$).
On CIFAR-100 and Tiny-ImageNet, smaller $\lambda$ is slightly preferred, yet gains persist
across a broad range.
We adopt $\lambda=0.1$ as a simple default.

\begin{table}[t]
\centering
\rowcolors{2}{white}{gray!30}
\caption{Effect of regularization weight $\lambda$ (top-1 accuracy \%).
$\lambda=0$: No Reg.}
\label{tab:lambda}
\setlength{\tabcolsep}{6pt}
\small
\begin{tabular}{c ccc}
\toprule
$\lambda$ & CIFAR-10 & CIFAR-100 & Tiny-ImageNet \\
\midrule
0.10 & 66.22 & 44.21 & 44.62 \\
0.05 & 66.19 & 43.72 & 44.92 \\
0.01 & \textbf{66.40} & 43.35 & 44.59 \\
0.005 & 65.51 & 43.16 & 44.55 \\
0.001 & 65.68 & \textbf{44.57} & \textbf{45.40} \\
\midrule
0\,(\textit{No Reg.}) & 65.30 & 42.82 & 43.93 \\
\bottomrule
\end{tabular}
\end{table}

\subsection{Supplementary Experiments}
\label{sec:supplementary}
We provide additional learning-dynamics analyses to complement the main results.
The global accuracy curves show that PushCen-ADFL follows stable convergence
trajectories across all datasets, supporting the final averaged accuracies reported
in Table~\ref{tab:accuracy}.
The delayed-client accuracy curves on a pseudo-time axis further show that
late-joining clients improve rapidly after entering the system and reach
performance comparable to online clients within a short pseudo-time window,
corroborating the delayed-client results summarized in Table~\ref{tab:delayed_clients}.
Together, these results reinforce the convergence stability and robustness of
PushCen-ADFL under asynchronous and heterogeneous decentralized settings.

\section{Conclusion}
This work investigated asynchronous decentralized federated learning, where frequent peer-to-peer, asymmetric and unbalanced communication and the statistical heterogeneity jointly induce high bandwidth cost, aggregation bias and model drift.
We proposed PushCen-ADFL, an asynchronous decentralized framework that couples communication, de-biased aggregation, and local stabilization in a shared centroid representation space, forming a closed loop between compression and optimization.
It integrates centroid-based model representation, average-preserving push-sum aggregation, centroid regularization, and bounded sender-deduplicated buffering to enable stable and robust training with delayed client participation.
Extensive experiments demonstrate that PushCen-ADFL consistently improves performance over existing asynchronous decentralized baselines while substantially reducing per-push communication cost, achieving a favorable accuracy-communication trade-off under dynamic settings.

\bibliographystyle{ACM-Reference-Format}
\bibliography{sample-base}

\newpage


\appendix

\onecolumn

\section{Appendix Overview}
\label{sec:app_overview}

The appendix provides complementary analyses and additional experimental results that support the main paper:

\begin{itemize}
    \item \textbf{Communication and computation properties of WCP.}
    Appendix~\ref{app:wcp_details} revisits Weight Clustering Pruning (WCP) and derives its per-push communication cost as the sum of centroid transmission and index assignment, yielding an interpretable compression ratio mainly controlled by the number of centroids $K$.
    It further compares the clustering overhead with standard training cost and shows that, under typical settings, WCP adds only a small fraction of extra computation.

    \item \textbf{Detailed convergence analysis of PushCen-ADFL.}
    Appendix~\ref{app:conv} presents a full convergence analysis under an event-driven asynchronous model on directed graphs.
    It formalizes push-sum aggregation with buffered (possibly stale) messages, centroid compression/reconstruction error, and the centroid proximal regularization used in local updates.
    The analysis establishes key lemmas on system-level mass conservation, numerator perturbation induced by lossy re-encoding, a consensus-error recursion under directed mixing with staleness and compression, and bounds showing how the centroid proximal term suppresses local deviation from its anchor.
    These components are combined into a constant-closed stationarity guarantee that makes the impacts of heterogeneity, staleness, directed mixing, and compression distortion explicit.

    \item \textbf{Supplementary experiments and ablations.}
    Appendix~\ref{app:supplementary_experiments} reports additional empirical evidence beyond the main tables, including global test-accuracy trajectories and delayed-client accuracy curves to visualize learning dynamics under asynchrony.
    It also provides ablations on the number of centroids $K$ and the regularization weight $\lambda$, quantifying the trade-offs between accuracy, per-push communication cost, and computational overhead, and validating the effectiveness and robustness of centroid regularization.
\end{itemize}

\section{Communication and Computational Analysis of Weight Clustering Pruning}
\label{app:wcp_details}

In this section, we provide a unified analysis of the communication and computational properties of Weight Clustering Pruning (WCP).
Although WCP was originally introduced in our previous work, it remains a core component of the proposed method and plays an important role in reducing both communication cost and practical system overhead.
For completeness, we revisit its theoretical properties and highlight why WCP remains efficient and well-suited for asynchronous federated learning.

\subsection{Communication Efficiency of Weight Clustering Pruning}
\label{subsec:wcp_comm}

We first analyze how WCP reduces communication overhead compared to standard federated learning.
In conventional federated optimization, each client transmits a full model parameter vector $\theta \in \mathbb{R}^N$ at every communication event.
Assuming each parameter is represented using $B$ bits, the communication cost per client per round is
\begin{equation}
C_{\text{full}} = N \times B \quad \text{bits}.
\end{equation}

Under WCP, model parameters are represented using a clustered form.
Specifically, weights are quantized into $k$ clusters and represented by a centroid set and an index assignment.
One centroid is fixed at zero to naturally encode pruned weights and therefore does not need to be transmitted.

The total communication cost consists of two components.
The first is the transmission of centroid values.
Only $k-1$ non-zero centroids need to be communicated, leading to
\begin{equation}
C_{\text{centroid}} = (k-1) \times B \quad \text{bits}.
\end{equation}

The second component corresponds to the index sequence that maps each weight to its associated centroid.
Each index requires $\lceil \log_2 k \rceil$ bits, resulting in
\begin{equation}
C_{\text{index}} = N \times \lceil \log_2 k \rceil \quad \text{bits}.
\end{equation}

Combining the two terms, the total communication cost of WCP is
\begin{equation}
C_{\text{WCP}} = (k-1) \times B + N \times \lceil \log_2 k \rceil.
\end{equation}

To quantify the communication reduction, we define the compression ratio
\begin{equation}
\rho = \frac{C_{\text{WCP}}}{C_{\text{full}}}
= \frac{(k-1)B + N \lceil \log_2 k \rceil}{N B}.
\end{equation}

Since modern neural networks typically satisfy $N \gg k$, the centroid transmission term becomes negligible.
Thus, the compression ratio can be approximated as
\begin{equation}
\rho \approx \frac{\lceil \log_2 k \rceil}{B}.
\end{equation}

This result shows that communication efficiency is primarily controlled by the number of clusters $k$.
A moderate choice of $k$ yields substantial bandwidth savings while maintaining sufficient representational capacity, which is essential for preserving model accuracy.

\subsection{Computational Overhead of Weight Clustering Pruning}
\label{subsec:wcp_comp}

We now examine the computational cost introduced by WCP and compare it to the standard training cost in federated learning.

Consider a neural network with $L$ layers, where the parameter matrix of layer $l$ contains $P^l = n_{l-1} n_l$ parameters.
The total number of parameters is
\begin{equation}
P_{\text{tot}} = \sum_{l=1}^{L} P^l.
\end{equation}

\paragraph{Training Cost.}
For each training sample, the forward pass, backward pass and parameter update all incur costs proportional to the number of parameters.
Thus, the per-sample training cost satisfies
\begin{equation}
C_{\text{train}} = \mathcal{O}\!\left(\sum_{l=1}^{L} P^l \right).
\end{equation}
For a local dataset of size $N$ and $E$ training epochs, the total training cost becomes
\begin{equation}
C_{\text{total\_train}} = \mathcal{O}\!\left(N E P_{\text{tot}}\right).
\end{equation}

\paragraph{Clustering Cost.}
WCP applies clustering independently to each layer.
For a layer with $P^l$ parameters, a K-means iteration consists of an assignment step with complexity $\mathcal{O}(P^l k)$ and a centroid update step with complexity $\mathcal{O}(P^l)$.
Assuming $T$ clustering iterations, the cost for layer $l$ is
\begin{equation}
C_{\text{cluster}}^{l} = \mathcal{O}(T k P^l).
\end{equation}
Summing over all layers yields
\begin{equation}
C_{\text{total\_cluster}} = \mathcal{O}(T k P_{\text{tot}}).
\end{equation}

\paragraph{Cost Comparison.}
The ratio between training cost and clustering cost is therefore
\begin{equation}
\frac{C_{\text{total\_train}}}{C_{\text{total\_cluster}}}
= \frac{N E}{T k}.
\end{equation}

Under typical experimental settings (e.g., $E=1$, $k=32$ and small $T$), this ratio is substantially larger than one.
This indicates that the computational overhead introduced by WCP is dominated by the standard training process and remains a small fraction of the overall computation.

\paragraph{Summary.}
Overall, WCP introduces negligible additional computation while providing significant communication savings.
Its computational cost scales linearly with the number of parameters and clusters, ensuring good scalability.
These properties make WCP particularly suitable for asynchronous and decentralized federated learning, where both communication efficiency and lightweight local computation are critical.

\section{Detailed convergence analysis of PushCen-ADFL}
\label{app:conv}
\subsection{Notation and Event-Driven Model}
\label{app:conv-notation}

\paragraph{Network and directed communication.}
We consider $N$ clients indexed by $\mathcal{V}=\{1,\dots,N\}$ connected by a directed graph
$\mathcal{G}=(\mathcal{V},\mathcal{E})$.
For client $i$, denote its in-neighbors and out-neighbors by $\mathcal{N}_i^{-}$ and $\mathcal{N}_i^{+}$ and let
$d_i^{+}\triangleq|\mathcal{N}_i^{+}|$.
Clients exchange messages only along directed edges; there is no central server.

\paragraph{Event-driven asynchronous timeline.}
Training evolves on a global \emph{event} timeline $t=0,1,2,\dots$.
At each event $t$, either (i) a message arrives at some client and is appended into its local buffer, or
(ii) a single client is \emph{activated} and performs an \emph{aggregate $\rightarrow$ local update $\rightarrow$ broadcast}
cycle.
Let $i_t\in\mathcal{V}$ denote the activated client at event $t$ (if any).
No global synchronization is assumed and different clients may be activated at different rates.

\paragraph{Local client states.}
Each client $i$ maintains the following states at event $t$:
\begin{itemize}
    \item \textbf{Model} $w_i^t\in\mathbb{R}^d$, the local trainable parameters (vectorized for analysis).
    \item \textbf{Push-sum mass} $y_i^t>0$, corresponding to \texttt{ps\_mass} in the implementation.
    \item \textbf{centroid dictionary} $G_i^t$, the layer-wise centroid tables maintained/updated by aggregation.
    \item \textbf{Buffer} $\mathcal{B}_i^t$, storing the most recent neighbor messages received since the last compute event.
\end{itemize}
\noindent\textbf{Remark (local-only parameters).}
In the implementation, BatchNorm-related parameters/statistics are \emph{local-only} (FedBN-style): they are not aggregated
and remain purely local. Our analysis focuses on the shared trainable vector $w_i^t$ (excluding such local-only components).

\paragraph{Centroid compression and reconstruction (implementation-aligned).}
When client $j$ transmits, it sends a centroid-compressed payload
\(
\mathcal{C}(w_j^t) \triangleq (V_j^t, A_j^t, U_j^t),
\)
where $V_j^t$ is the centroid table, $A_j^t$ is the assignment map and $U_j^t$ contains \emph{uncompressed shared tensors}
that are communicated without clustering (e.g., biases or non-compressible layers).
Local-only parameters (e.g., BatchNorm statistics) are \emph{not} transmitted.
Upon reception, client $i$ reconstructs an approximate model
\(
\hat{w}_j^t = \mathcal{R}(\mathcal{C}(w_j^t)),
\)
where $\mathcal{R}$ indexes $V_j^t$ with $A_j^t$ and copies $U_j^t$.
We define the reconstruction (compression) error:
\begin{equation}
e_j^t \triangleq \hat{w}_j^t - w_j^t.
\label{eq:compression-error-def}
\end{equation}

\paragraph{Buffered messages.}
A message sent from $j$ to $i$ is represented as
\(
m_{j\to i}^t = (\hat{w}_j^t, V_j^t, y_{j\to i}^t, j),
\)
where $y_{j\to i}^t>0$ is the push-sum mass share attached to the message (field \texttt{ps\_mass\_share}).
At event $t$, client $i$'s buffer is
\(
\mathcal{B}_i^t=\{m_{j\to i}^t: j\in\mathcal{S}_i^t\},
\)
where $\mathcal{S}_i^t\subseteq\mathcal{N}_i^{-}$ indexes the set of senders currently stored.
The implementation keeps only the newest message per sender (sender deduplication) and may enforce a capacity bound; in analysis
we treat $\mathcal{B}_i^t$ as an arbitrary finite set of most-recent neighbor messages.

\paragraph{Push-sum reformulation (numerator/denominator variables).}
To make average-preservation explicit under directed, potentially unbalanced communication, we introduce the standard push-sum
\emph{numerator} variable
\begin{equation}
x_i^t \triangleq y_i^t\, w_i^t \in \mathbb{R}^d.
\label{eq:pushsum-x-def}
\end{equation}
The de-biased local estimate is always recovered as $w_i^t = x_i^t / y_i^t$.

\paragraph{Mass splitting (code-consistent).}
After finishing a local update, client $i$ splits its current mass uniformly among its out-neighbors and itself:
\begin{equation}
y_{i\to k}^t \;=\; \frac{y_i^t}{d_i^{+}+1},\quad \forall k\in\mathcal{N}_i^{+},
\qquad
y_i^{t+} \;=\; \frac{y_i^t}{d_i^{+}+1}.
\label{eq:mass-splitting}
\end{equation}
Each outgoing message to $k$ carries $y_{i\to k}^t$ and client $i$ keeps the same share locally.
Hence, the sender-side conservation holds \emph{exactly}:
\(
d_i^{+}\cdot \frac{y_i^t}{d_i^{+}+1} + 1\cdot \frac{y_i^t}{d_i^{+}+1} = y_i^t.
\)
Analogously, the transmitted numerator share is $x_{i\to k}^t \triangleq y_{i\to k}^t \hat{w}_i^t$ and the locally retained
numerator is $x_i^{t+}\triangleq y_i^{t+} w_i^{t+}$.

\paragraph{System-level accounting for in-flight/buffered mass.}
Under asynchronous message passing, a portion of mass may reside in buffers (or be in transit) between send and aggregation events.
Let $\mathcal{M}^t$ denote the set of all messages that are currently buffered or in transit in the system at event $t$.
For each message $m\in\mathcal{M}^t$, let $y_m$ be its mass share and $x_m$ its numerator share.
We define the \emph{total system mass} and \emph{total system numerator}:
\begin{equation}
Y_{\mathrm{tot}}^t \triangleq \sum_{i=1}^{N} y_i^t + \sum_{m\in\mathcal{M}^t} y_m,
\qquad
X_{\mathrm{tot}}^t \triangleq \sum_{i=1}^{N} x_i^t + \sum_{m\in\mathcal{M}^t} x_m .
\label{eq:total-mass-def}
\end{equation}
The mass-splitting rule~\eqref{eq:mass-splitting} and the aggregation rule (which transfers buffered shares into the local
states) together imply that $(Y_{\mathrm{tot}}^t, X_{\mathrm{tot}}^t)$ are invariant over time; this will be formalized
in the mass-conservation lemma.

\paragraph{De-biased global reference sequence.}
Using the conserved system totals, we define the canonical de-biased global reference model:
\begin{equation}
\bar{w}^t \triangleq \frac{X_{\mathrm{tot}}^t}{Y_{\mathrm{tot}}^t}.
\label{eq:global-reference}
\end{equation}
This sequence is well-defined on asymmetric and unbalanced graphs and serves as the main reference trajectory for optimality
analysis (while $w_i^t$ are local de-biased estimates with a consensus error).

\paragraph{Local objective and centroid proximal target.}
Client $i$ holds a private dataset $\mathcal{D}_i$ and local objective
\(
f_i(w)\triangleq \mathbb{E}_{\xi\sim\mathcal{D}_i}[\ell(w;\xi)].
\)
Before each compute event, client $i$ constructs a proximal target $\tilde{w}_i^t$ from its current centroid dictionary estimate $G_i^t$
and its current assignment map (obtained by clustering the current local model) and performs local stochastic optimization on
\begin{equation}
f_i(w) + \lambda \|w-\tilde{w}_i^t\|_2^2,
\label{eq:local-prox-objective}
\end{equation}
where $\lambda\ge 0$ is the regularization weight.
The construction of $\tilde{w}_i^t$ is implementation-aligned: it is produced by indexing centroid tables using the assignment
map so that $\tilde{w}_i^t$ matches the shape of $w_i^t$.

\subsection{Assumptions}
\label{sec:conv-assumptions}

We list the standard assumptions required to establish the convergence of PushCen-ADFL.
These assumptions cover stochastic optimization with heterogeneous data distributions,
asynchronous communication with message delays and centroid-based lossy communication over directed graphs.
All random variables are defined on a common probability space induced by data sampling,
client activations and message arrivals.

\paragraph{(A1) Smoothness and lower boundedness.}
Each local objective $f_i:\mathbb{R}^d\rightarrow\mathbb{R}$ is $L$-smooth:
\begin{equation}
\|\nabla f_i(u)-\nabla f_i(v)\|_2 \le L\|u-v\|_2,\quad \forall u,v\in\mathbb{R}^d,
\label{eq:assump-smooth}
\end{equation}
and the global objective $F(w)\triangleq \frac{1}{N}\sum_{i=1}^N f_i(w)$ is lower bounded,
i.e., $F(w)\ge F^\star$ for some finite $F^\star$.

\paragraph{(A2) Unbiased stochastic gradients with bounded variance.}
At client $i$, a stochastic gradient computed from a mini-batch $\xi\sim\mathcal{D}_i$
satisfies
\begin{equation}
\mathbb{E}\!\left[g_i(w;\xi)\mid w\right] = \nabla f_i(w),
\qquad
\mathbb{E}\!\left[\|g_i(w;\xi)-\nabla f_i(w)\|_2^2 \mid w\right] \le \sigma^2,
\label{eq:assump-sgd}
\end{equation}
for some $\sigma^2<\infty$.
When centroid proximal regularization is applied, we equivalently view the client as optimizing the
regularized objective $f_i^\lambda(w)\triangleq f_i(w)+\lambda\|w-\tilde w_i\|_2^2$,
and the stochastic gradient is taken w.r.t.\ $f_i^\lambda$.

\paragraph{(A3) Uniform client heterogeneity.}
There exists a constant $\zeta^2<\infty$ such that for all $w\in\mathbb{R}^d$ and all clients $i\in[N]$,
\begin{equation}
\|\nabla f_i(w)-\nabla F(w)\|_2^2 \le \zeta^2,
\qquad\text{where}\qquad
F(w)\triangleq \frac{1}{N}\sum_{j=1}^{N} f_j(w).
\label{eq:A3-uniform}
\end{equation}

\noindent\textbf{Remark.}
Assumption A3 is a standard strengthening of the mean-heterogeneity condition and is commonly used
in non-convex FL analyses to control one-step descent for the activated client under asynchrony.

\paragraph{(A4) Bounded asynchrony (message staleness).}
Messages aggregated at event $t$ may correspond to stale model states due to heterogeneous computation/communication.
We assume there exists an integer $\tau\ge 0$ such that any message used by an activated client at event $t$
was generated within the last $\tau$ activations/events.
Equivalently, if a buffered message $m_{j\to i}^t$ carries a reconstruction $\hat w_j^{\kappa}$,
then its generation time $\kappa$ satisfies
\begin{equation}
t-\kappa \le \tau.
\label{eq:assump-stale}
\end{equation}

\paragraph{(A5) Directed mixing and push-sum well-posedness.}
The directed communication graph is strongly connected and the sequence of effective communication events
is sufficiently mixing.
Concretely, there exists an integer $B\ge 1$ such that for any window of $B$ consecutive events,
the union of directed edges along which messages are delivered is strongly connected.
Moreover, the push-sum masses are uniformly bounded away from zero:
there exists $y_{\min}>0$ such that
\begin{equation}
y_i^t \ge y_{\min},\quad \forall i\in\mathcal{V},\ \forall t.
\label{eq:assump-mass-lb}
\end{equation}
This ensures the de-biased ratios $w_i^t = x_i^t/y_i^t$ and the global reference $\bar w^t$ are well-defined.

\paragraph{(A6) Bounded absolute compression error (C1).}
The centroid compression-reconstruction operator $\mathcal{R}\circ\mathcal{C}$ induces a bounded absolute error:
there exists $\varepsilon_c\ge 0$ such that for any communicated model $w$,
\begin{equation}
\|\mathcal{R}(\mathcal{C}(w)) - w\|_2 \le \varepsilon_c.
\label{eq:assump-compress-C1}
\end{equation}
Equivalently, with $e_j^t$ defined in~\eqref{eq:compression-error-def}, we have $\|e_j^t\|_2 \le \varepsilon_c$
for all transmitted messages.
This assumption captures the deterministic quantization/clustering distortion introduced by WCP with a fixed
number of centroids $K$.

\noindent\textbf{Remark.}
Assumptions (A1)-(A3) are standard in stochastic non-convex optimization and federated learning.
Assumptions (A4)-(A5) formalize bounded staleness and directed-graph mixing needed for asynchronous push-sum.
Assumption (A6) matches our implementation, where WCP-based centroid encoding introduces deterministic reconstruction
distortion controlled by the clustering resolution (e.g., $K=32$).

\subsection{Update Rules}
\label{sec:conv-updaterule}

We formalize the event-driven update rules of PushCen-ADFL (ADFLCenReg) by separating
(i) asynchronous push-sum aggregation over buffered neighbor messages and
(ii) local stochastic optimization with centroid proximal regularization.
Throughout, we use the push-sum numerator-denominator representation
$x_i^t = y_i^t w_i^t$ introduced in~\eqref{eq:pushsum-x-def}.

\paragraph{Message format and staleness-aware notation.}
When client $j$ is activated and broadcasts to an out-neighbor $i\in\mathcal{N}_j^{+}$,
it transmits a tuple
\begin{equation}
m_{j\to i}^{(\kappa)} \;=\; \big(\mathcal{C}(w_j^{\kappa}),\; y_{j\to i}^{(\kappa)},\; j \big),
\label{eq:message-format}
\end{equation}
where $\kappa$ denotes the \emph{generation event index} at which the message is produced,
$\mathcal{C}(w_j^{\kappa})=(V_j^{\kappa},A_j^{\kappa},U_j^{\kappa})$ is the centroid-compressed payload and
$y_{j\to i}^{(\kappa)}>0$ is the attached mass share.
Upon reception at some later event $t\ge \kappa$, client $i$ reconstructs the (possibly stale) neighbor model
\begin{equation}
\hat{w}_{j\to i}^t \;\triangleq\; \mathcal{R}\!\left(\mathcal{C}(w_j^{\kappa})\right),
\label{eq:stale-recon}
\end{equation}
and buffers the tuple $(\hat{w}_{j\to i}^t, V_j^{\kappa}, y_{j\to i}^{(\kappa)}, j)$.
\footnote{Here $\hat{w}_{j\to i}^t$ denotes the content \emph{available at client $i$ at event $t$}.
Due to asynchrony, it generally represents a stale state of client $j$, i.e., $\hat{w}_{j\to i}^t\approx w_j^{t-\tau}$
under Assumption~\ref{eq:assump-stale}.}
At the time client $i$ is activated at event $t$, let $\mathcal{B}_i^t$ denote its buffer,
i.e., the set of currently stored messages to be aggregated (sender-deduplicated and possibly capacity-bounded).

\paragraph{C1. Asynchronous push-sum aggregation (buffer processing).}
When client $i=i_t$ is activated at event $t$, it aggregates all buffered neighbor messages
$\mathcal{B}_i^t = \{(\hat w_{j\to i}^t, V_j^{\kappa}, y_{j\to i}^{(\kappa)}, j)\}$.
Define the post-buffer total mass at $i$:
\begin{equation}
Y_i^t \;\triangleq\; y_i^t + \sum_{(\cdot,\cdot,y_{j\to i}^{(\kappa)},\cdot)\in\mathcal{B}_i^t} y_{j\to i}^{(\kappa)}.
\label{eq:agg-total-mass}
\end{equation}
The de-biased model update is
\begin{equation}
w_i^{t+\frac{1}{3}}
\;=\;
\frac{y_i^t}{Y_i^t}\, w_i^t
\;+\;
\sum_{(\hat w_{j\to i}^t,\cdot,y_{j\to i}^{(\kappa)},\cdot)\in\mathcal{B}_i^t}
\frac{y_{j\to i}^{(\kappa)}}{Y_i^t}\, \hat w_{j\to i}^t,
\label{eq:agg-model}
\end{equation}
with the mass update and buffer clearing
\begin{equation}
y_i^{t+\frac{1}{3}} \;=\; Y_i^t,
\qquad
\mathcal{B}_i^{t+\frac{1}{3}} \;=\; \emptyset.
\label{eq:agg-mass}
\end{equation}
Equivalently, in numerator-denominator form, letting $x_i^t=y_i^t w_i^t$ and
$x_{j\to i}^{(\kappa)}\triangleq y_{j\to i}^{(\kappa)} \hat w_{j\to i}^t$, the aggregation is additive:
\begin{equation}
x_i^{t+\frac{1}{3}} \;=\; x_i^t + \sum_{(\cdot,\cdot,y_{j\to i}^{(\kappa)},\cdot)\in\mathcal{B}_i^t} x_{j\to i}^{(\kappa)},
\qquad
y_i^{t+\frac{1}{3}} \;=\; y_i^t + \sum_{(\cdot,\cdot,y_{j\to i}^{(\kappa)},\cdot)\in\mathcal{B}_i^t} y_{j\to i}^{(\kappa)}.
\label{eq:agg-xy}
\end{equation}
Finally $w_i^{t+\frac{1}{3}} = x_i^{t+\frac{1}{3}} / y_i^{t+\frac{1}{3}}$.

\paragraph{centroid dictionary aggregation.}
Client $i$ updates its centroid dictionary estimate using the same push-sum weights:
\begin{equation}
G_i^{t+\frac{1}{3}}
\;=\;
\frac{y_i^t}{Y_i^t}\, G_i^t
\;+\;
\sum_{(\cdot,V_j^{\kappa},y_{j\to i}^{(\kappa)},\cdot)\in\mathcal{B}_i^t}
\frac{y_{j\to i}^{(\kappa)}}{Y_i^t}\, V_j^{\kappa}.
\label{eq:agg-pool}
\end{equation}
(We treat $G_i^t$ as an auxiliary state whose role is to form the proximal target in the local update.)

\paragraph{C2. Local stochastic optimization with centroid proximal regularization.}
After aggregation, client $i$ performs $E$ local epochs/steps of SGD on
\begin{equation}
\Phi_i^t(w) \;\triangleq\; f_i(w) + \lambda \|w - \tilde w_i^t\|_2^2,
\label{eq:local-phi}
\end{equation}
where $\tilde w_i^t$ is the centroid-based proximal target constructed from the current centroid dictionary $G_i^{t+\frac{1}{3}}$
and the current assignment map (obtained by applying WCP to $w_i^{t+\frac{1}{3}}$).
Let $(V_i^t,A_i^t,M_i^t)$ be the centroid table, assignment map and pruning mask produced by WCP at the beginning of the local update.
For $e=0,\dots,E-1$, we write the masked proximal-SGD step as

\begin{equation}
\begin{aligned}
w_{i,e+1}
&=
\Pi_{M_i^t}\!\Bigl(
    w_{i,e}
    - \eta\, g_i\!\left(w_{i,e};\xi_{i,e}\right) \\
&\qquad\quad
    - 2\eta\lambda\,(w_{i,e}-\tilde w_i^t)
\Bigr),
\qquad
w_{i,0} = w_i^{t+\frac{1}{3}} .
\end{aligned}
\label{eq:local-sgd}
\end{equation}

where $\xi_{i,e}$ is a mini-batch sampled from $\mathcal{D}_i$ and $g_i(\cdot;\xi)$ is the stochastic gradient of $\ell(\cdot;\xi)$.
The masking operator $\Pi_{M_i^t}(\cdot)$ enforces pruning by element-wise multiplication,
\begin{equation}
\Pi_{M_i^t}(u) \triangleq u \odot M_i^t,
\label{eq:mask-proj}
\end{equation}
matching the implementation (\texttt{w $\leftarrow$ w * M}).
We denote the post-local-update model by
\begin{equation}
w_i^{t+\frac{2}{3}} \;\triangleq\; w_{i,E}.
\label{eq:local-output}
\end{equation}

\paragraph{C3. Post-training re-encoding and broadcast (mass splitting).}
After local optimization, client $i$ re-encodes its model via WCP to obtain an updated centroid representation
$\mathcal{C}(w_i^{t+\frac{2}{3}})=(V_i^{t+\frac{2}{3}},A_i^{t+\frac{2}{3}},U_i^{t+\frac{2}{3}})$ to be transmitted.
It then performs push-sum mass splitting:
\begin{equation}
y_{i\to k}^{t+\frac{2}{3}} \;=\; \frac{y_i^{t+\frac{1}{3}}}{d_i^{+}+1},\quad \forall k\in\mathcal{N}_i^{+},
\qquad
y_i^{t+1} \;=\; \frac{y_i^{t+\frac{1}{3}}}{d_i^{+}+1},
\label{eq:broadcast-mass}
\end{equation}
and sends $(\mathcal{C}(w_i^{t+\frac{2}{3}}), y_{i\to k}^{t+\frac{2}{3}}, i)$ to each $k\in\mathcal{N}_i^{+}$.
Finally,
\begin{equation}
w_i^{t+1} \;\triangleq\; w_i^{t+\frac{2}{3}},
\qquad
x_i^{t+1} \;\triangleq\; y_i^{t+1} w_i^{t+1}.
\label{eq:end-of-event}
\end{equation}
For any client $\ell\neq i_t$ that is not activated at event $t$, we keep its local state unchanged:
$(w_\ell^{t+1},x_\ell^{t+1},y_\ell^{t+1},G_\ell^{t+1})=(w_\ell^{t},x_\ell^{t},y_\ell^{t},G_\ell^{t})$.

\subsection{Key Quantities}
\label{sec:conv-key-quantities}

To analyze PushCen-ADFL under directed asynchronous communication, we introduce a de-biased global reference sequence
together with two key error measures: an \emph{optimality} metric that tracks progress toward stationary points and a
\emph{consensus} metric that tracks network disagreement.

\paragraph{De-biased global reference (push-sum average).}
Recall the total system numerator and mass defined in~\eqref{eq:total-mass-def}:
$X_{\mathrm{tot}}^t$ and $Y_{\mathrm{tot}}^t$.
We define the (virtual) de-biased global reference model as
\begin{equation}
\bar{w}^t \triangleq \frac{X_{\mathrm{tot}}^t}{Y_{\mathrm{tot}}^t}.
\label{eq:key-wbar}
\end{equation}
Unlike the naive arithmetic mean $\frac{1}{N}\sum_{i=1}^N w_i^t$, $\bar{w}^t$ is the correct average-preserving reference
on directed and potentially unbalanced graphs under push-sum dynamics.

\paragraph{Local de-biased models (node states).}
Each client maintains a local push-sum numerator $x_i^t$ and mass $y_i^t$.
The local de-biased model held at node $i$ is
\begin{equation}
w_i^t \triangleq \frac{x_i^t}{y_i^t}.
\label{eq:key-local-debias}
\end{equation}
\textbf{Important:} We strictly distinguish the local model $w_i^t$ from the centroid proximal anchor $\tilde{w}_i^t$
defined in~\eqref{eq:local-phi}, where $\tilde{w}_i^t$ serves as the regularization target during local updates.

\paragraph{Consensus error (network disagreement).}
We quantify network disagreement by the mean-squared deviation of local models from the de-biased global reference:
\begin{equation}
\mathcal{E}_{\mathrm{con}}^t
\triangleq
\frac{1}{N}\sum_{i=1}^{N}\left\|w_i^t - \bar{w}^t\right\|_2^2.
\label{eq:key-consensus}
\end{equation}
This term captures the combined effects of directed mixing, asynchronous staleness and lossy reconstruction on consensus.

\paragraph{Optimality (stationarity) measure.}
For the global objective $F(w)\triangleq \frac{1}{N}\sum_{i=1}^{N} f_i(w)$, we use the standard non-convex stationarity metric
evaluated at the de-biased global reference sequence:
\begin{equation}
\mathcal{E}_{\mathrm{opt}}^t
\triangleq
\mathbb{E}\left[\left\|\nabla F(\bar{w}^t)\right\|_2^2\right].
\label{eq:key-optimality}
\end{equation}
Bounding $\frac{1}{T}\sum_{t=0}^{T-1}\mathcal{E}_{\mathrm{opt}}^t$ implies convergence to a stationary neighborhood.

\noindent\textbf{Remark.}
Our analysis proceeds by (i) deriving a recursion for $\mathcal{E}_{\mathrm{con}}^t$ under push-sum mixing with perturbations
(compression and staleness) and (ii) proving a descent inequality for $F(\bar{w}^t)$ whose error terms are controlled by
$\mathcal{E}_{\mathrm{con}}^t$. Combining the two yields a bound on the averaged stationarity measure in~\eqref{eq:key-optimality}.

\subsection{Key Lemmas}
\label{sec:conv-key-lemmas}

\subsubsection{Lemma 1}
\begin{lemma}[System-level mass conservation]
\label{lem:mass-conservation}
Recall the set of in-flight/buffered messages $\mathcal{M}^t$ and the total system mass
\begin{equation}
Y_{\mathrm{tot}}^t \triangleq \sum_{i=1}^{N} y_i^t + \sum_{m\in\mathcal{M}^t} y_m,
\label{eq:total-mass-recall}
\end{equation}
where $y_i^t$ is the local push-sum mass at node $i$ and $y_m$ is the mass share carried by message $m$.
Under the update rules in Section~\ref{sec:conv-updaterule}, the total system mass is invariant:
\begin{equation}
Y_{\mathrm{tot}}^{t+1} = Y_{\mathrm{tot}}^{t},\qquad \forall t\ge 0.
\label{eq:mass-conservation-only}
\end{equation}
Consequently, the de-biased global reference $\bar w^t = X_{\mathrm{tot}}^t / Y_{\mathrm{tot}}^t$ is well-defined for all $t$
since $Y_{\mathrm{tot}}^t$ is constant and strictly positive (Assumption~\eqref{eq:assump-mass-lb}).
\end{lemma}
\begin{proof}
We show that $Y_{\mathrm{tot}}$ does not change under either type of event: \emph{(i) message generation/sending} and
\emph{(ii) message aggregation/consumption}. All other operations (e.g., local SGD on $w_i$) do not modify $y_i$ and thus
do not affect $Y_{\mathrm{tot}}$.

\paragraph{(i) Message generation/sending (mass splitting).}
Suppose client $i$ is activated and performs mass splitting as in~\eqref{eq:broadcast-mass}.
Let $d_i^{+}=|\mathcal{N}_i^{+}|$ and denote the pre-splitting local mass by $y_i$.
The splitting rule sets
\[
y_i' \;=\; \frac{y_i}{d_i^{+}+1},
\qquad
y_{i\to k} \;=\; \frac{y_i}{d_i^{+}+1},\ \forall k\in\mathcal{N}_i^{+},
\]
and creates $d_i^{+}$ new outgoing messages, each carrying mass $y_{i\to k}$.
Let $\mathcal{M}$ be the message set immediately before splitting and $\mathcal{M}'$ the set immediately after splitting.
Then $\mathcal{M}' = \mathcal{M}\cup \{m_{i\to k}:k\in\mathcal{N}_i^{+}\}$ and we have
\begin{align}
\Big(y_i' + \sum_{m\in\mathcal{M}'} y_m\Big) - \Big(y_i + \sum_{m\in\mathcal{M}} y_m\Big)
&=
(y_i' - y_i) + \sum_{k\in\mathcal{N}_i^{+}} y_{i\to k} \nonumber\\
&=
\left(\frac{y_i}{d_i^{+}+1}-y_i\right)
+
d_i^{+}\cdot \frac{y_i}{d_i^{+}+1} \nonumber\\
&= 0.
\label{eq:mass-split-diff}
\end{align}
All other nodes $\ell\neq i$ keep $y_\ell$ unchanged during this event, hence $Y_{\mathrm{tot}}$ is invariant.

\paragraph{(ii) Message aggregation/consumption.}
When an activated client $i$ aggregates its buffer (Eq.~\eqref{eq:agg-mass}), it consumes a subset of messages
$\mathcal{B}_i \subseteq \mathcal{M}$ and transfers their mass shares into the local state:
\[
y_i' \;=\; y_i + \sum_{m\in\mathcal{B}_i} y_m,
\qquad
\mathcal{M}' \;=\; \mathcal{M}\setminus \mathcal{B}_i.
\]
Therefore,
\begin{align}
\Big(y_i' + \sum_{m\in\mathcal{M}'} y_m\Big) - \Big(y_i + \sum_{m\in\mathcal{M}} y_m\Big)
&=
\sum_{m\in\mathcal{B}_i} y_m \;-\; \sum_{m\in\mathcal{B}_i} y_m
= 0.
\label{eq:mass-agg-diff}
\end{align}
Again, all other nodes keep their masses unchanged, so $Y_{\mathrm{tot}}$ remains invariant.

\paragraph{Conclusion.}
Since $Y_{\mathrm{tot}}$ is unchanged by both message splitting and message aggregation, we obtain
$Y_{\mathrm{tot}}^{t+1}=Y_{\mathrm{tot}}^{t}$ for all events $t$.
\end{proof}

\noindent\textbf{Remark.}
In contrast to the mass variable, the system numerator $X_{\mathrm{tot}}^t$ may be perturbed by lossy compression:
in the implementation, a sender retains the full-precision local model $w_i$ while transmitting a compressed representation
that reconstructs to $\hat w_i\neq w_i$ at receivers. This perturbation will be explicitly handled in subsequent lemmas
when deriving consensus and optimality recursions under Assumption~\eqref{eq:assump-compress-C1}.

\subsubsection{Lemma 2}
\begin{lemma}[Numerator perturbation induced by lossy compression]
\label{lem:numerator-perturbation}
Recall the total system mass $Y_{\mathrm{tot}}^t$ in~\eqref{eq:total-mass-recall}, which is invariant by
Lemma~\ref{lem:mass-conservation}. Let
\begin{equation}
X_{\mathrm{tot}}^t \triangleq \sum_{i=1}^{N} x_i^t + \sum_{m\in\mathcal{M}^t} x_m
\label{eq:total-numerator-recall}
\end{equation}
be the total system numerator, where $x_i^t \triangleq y_i^t w_i^t$ and each message $m$ carries numerator share
$x_m \triangleq y_m \hat w_m$ with reconstructed content $\hat w_m$.

Consider an activation event $t$ at client $i=i_t$.
Let $w_i^{t+\frac{2}{3}}$ be the post-local-update model before broadcasting
(Section~\ref{sec:conv-updaterule}, C2) and define the \emph{message model} produced by centroid re-encoding as
\begin{equation}
\check w_i^{t} \triangleq \mathcal{R}\!\left(\mathcal{C}(w_i^{t+\frac{2}{3}})\right),
\qquad
e_i^{t} \triangleq \check w_i^{t} - w_i^{t+\frac{2}{3}}.
\label{eq:sender-quant-error}
\end{equation}
Then the only communication operation that changes $X_{\mathrm{tot}}$ is the mass splitting / broadcast step,
and the induced numerator change satisfies the exact identity
\begin{equation}
X_{\mathrm{tot}}^{t+1} - X_{\mathrm{tot}}^{t+\frac{2}{3}}
\;=\;
\sum_{k\in\mathcal{N}_i^{+}} y_{i\to k}^{t+\frac{2}{3}}\, e_i^{t}.
\label{eq:Xtot-delta-exact}
\end{equation}
Under Assumption~(A6) (bounded absolute compression error), i.e.,
$\|e_i^{t}\|_2\le \varepsilon_c$, we further have
\begin{equation}
\big\|X_{\mathrm{tot}}^{t+1} - X_{\mathrm{tot}}^{t+\frac{2}{3}}\big\|_2
\;\le\;
\Big(\sum_{k\in\mathcal{N}_i^{+}} y_{i\to k}^{t+\frac{2}{3}}\Big)\,\varepsilon_c
\;=\;
\frac{d_i^{+}}{d_i^{+}+1}\,y_i^{t+\frac{1}{3}}\,\varepsilon_c,
\label{eq:Xtot-delta-bound}
\end{equation}
where the last equality uses the mass splitting rule~\eqref{eq:broadcast-mass}.
Consequently, since $Y_{\mathrm{tot}}$ is invariant, the de-biased global reference
$\bar w^t = X_{\mathrm{tot}}^t/Y_{\mathrm{tot}}^t$ satisfies
\begin{equation}
\big\|\bar w^{t+1} - \bar w^{t+\frac{2}{3}}\big\|_2
\;=\;
\frac{\big\|X_{\mathrm{tot}}^{t+1} - X_{\mathrm{tot}}^{t+\frac{2}{3}}\big\|_2}{Y_{\mathrm{tot}}}
\;\le\;
\frac{d_i^{+}}{d_i^{+}+1}\cdot \frac{y_i^{t+\frac{1}{3}}}{Y_{\mathrm{tot}}}\,\varepsilon_c.
\label{eq:wbar-delta-bound}
\end{equation}
\end{lemma}

\begin{proof}
We focus on the broadcast (mass splitting) step, since: (i) message arrivals only move a message into some buffer and do not
change $\mathcal{M}^t$ as a multiset of in-flight shares and (ii) buffer aggregation transfers message shares from
$\mathcal{M}^t$ into a local state without changing the total numerator (it is a pure re-allocation of existing shares).
Therefore, the only communication-induced change to $X_{\mathrm{tot}}$ arises from creating new outgoing messages while
scaling down the sender's local mass.

Fix event $t$ where client $i$ broadcasts after local training.
Let the local mass right after aggregation be $y_i^{t+\frac{1}{3}}$ (Eq.~\eqref{eq:agg-mass}), which is the mass to be split.
Before splitting, the sender contributes local numerator
\(
x_i^{\mathrm{pre}} = y_i^{t+\frac{1}{3}}\, w_i^{t+\frac{2}{3}}
\)
to the system total and the message set does not contain the newly created outgoing shares.
After splitting (Eq.~\eqref{eq:broadcast-mass}), the sender retains mass
\(
y_i^{t+1} = \frac{y_i^{t+\frac{1}{3}}}{d_i^{+}+1}
\)
while creating $d_i^{+}$ outgoing messages each with mass
\(
y_{i\to k}^{t+\frac{2}{3}}=\frac{y_i^{t+\frac{1}{3}}}{d_i^{+}+1}.
\)
Crucially, the implementation sends a centroid-compressed payload, whose receiver-side reconstruction equals the
message model $\check w_i^{t}$ in~\eqref{eq:sender-quant-error}.
Hence, immediately after splitting:
\begin{itemize}
\item the sender's retained local numerator becomes
$
x_i^{\mathrm{post}} = y_i^{t+1}\, w_i^{t+\frac{2}{3}},
$
because the local model is \emph{not} replaced by its compressed reconstruction;
\item the newly created message set contributes additional numerator
$
\sum_{k\in\mathcal{N}_i^{+}} y_{i\to k}^{t+\frac{2}{3}}\, \check w_i^{t}.
$
\end{itemize}

Therefore, the net change in the system total numerator caused by splitting is
\begin{align}
X_{\mathrm{tot}}^{t+1} - X_{\mathrm{tot}}^{t+\frac{2}{3}}
&=
\Big(x_i^{\mathrm{post}} + \sum_{k\in\mathcal{N}_i^{+}} y_{i\to k}^{t+\frac{2}{3}}\,\check w_i^{t}\Big)
- x_i^{\mathrm{pre}} \nonumber\\
&=
\Big(\tfrac{y_i^{t+\frac{1}{3}}}{d_i^{+}+1} w_i^{t+\frac{2}{3}}
+ d_i^{+}\cdot \tfrac{y_i^{t+\frac{1}{3}}}{d_i^{+}+1}\check w_i^{t}\Big)
- y_i^{t+\frac{1}{3}} w_i^{t+\frac{2}{3}} \nonumber\\
&=
\frac{d_i^{+}}{d_i^{+}+1}\, y_i^{t+\frac{1}{3}}\big(\check w_i^{t}-w_i^{t+\frac{2}{3}}\big) \nonumber\\
&=
\sum_{k\in\mathcal{N}_i^{+}} y_{i\to k}^{t+\frac{2}{3}}\, e_i^{t},
\label{eq:deltaX-derivation}
\end{align}
which is exactly~\eqref{eq:Xtot-delta-exact}.

Taking norms and using $\|e_i^{t}\|_2\le \varepsilon_c$ yields
\[
\big\|X_{\mathrm{tot}}^{t+1} - X_{\mathrm{tot}}^{t+\frac{2}{3}}\big\|_2
\le
\sum_{k\in\mathcal{N}_i^{+}} y_{i\to k}^{t+\frac{2}{3}}\, \|e_i^{t}\|_2
\le
\Big(\sum_{k\in\mathcal{N}_i^{+}} y_{i\to k}^{t+\frac{2}{3}}\Big)\varepsilon_c,
\]
and substituting $\sum_{k} y_{i\to k}^{t+\frac{2}{3}} = \frac{d_i^{+}}{d_i^{+}+1}y_i^{t+\frac{1}{3}}$
gives~\eqref{eq:Xtot-delta-bound}.
Finally, Lemma~\ref{lem:mass-conservation} implies $Y_{\mathrm{tot}}^{t+1}=Y_{\mathrm{tot}}^{t}=Y_{\mathrm{tot}}$,
so dividing by the constant $Y_{\mathrm{tot}}$ yields~\eqref{eq:wbar-delta-bound}.
\end{proof}

\subsubsection{Lemma 3}
\begin{lemma}[Consensus error recursion under asynchronous push-sum with perturbations]
\label{lem:consensus-recursion}
Recall the consensus error
\(
\mathcal{E}_{\mathrm{con}}^t = \frac{1}{N}\sum_{i=1}^N \|w_i^t-\bar w^t\|_2^2
\)
in~\eqref{eq:key-consensus}, where $w_i^t=x_i^t/y_i^t$ and $\bar w^t=X_{\mathrm{tot}}^t/Y_{\mathrm{tot}}^t$.
Consider an activation event $t$ at client $i=i_t$.

Define the activated-client drift
\begin{equation}
\Delta_i^t \triangleq w_i^{t+\frac{2}{3}} - w_i^{t+\frac{1}{3}},
\label{eq:drift-def-lem3}
\end{equation}
and for any buffered message from sender $j$ generated at event $\kappa$ and available at client $i$ at event $t$, define the
receiver-side reconstruction error
\begin{equation}
e_{j\to i}^{(\kappa)} \triangleq \hat w_{j\to i}^{t} - w_j^{\kappa},
\qquad \|e_{j\to i}^{(\kappa)}\|_2 \le \varepsilon_c \ \ \text{by Assumption (A6)}.
\label{eq:recv-recon-error}
\end{equation}

Under Assumptions (A4)-(A6) and the directed mixing Assumption (A5), there exist constants $\rho\in(0,1)$ and
$C_1,C_2,C_3>0$ such that
\begin{equation}
\begin{aligned}
\mathbb{E}\!\left[\mathcal{E}_{\mathrm{con}}^{t+1}\right]
\;\le\;
&\rho\,\mathbb{E}\!\left[\mathcal{E}_{\mathrm{con}}^{t}\right] + C_1\,\mathbb{E}\!\left[\|\Delta_{i_t}^t\|_2^2\right] + C_2\,\tau
\sum_{r=t-\tau}^{t-1}
\mathbb{E}\!\left[\|\Delta_{i_r}^r\|_2^2\right] + C_3\,\varepsilon_c^2 .
\end{aligned}
\label{eq:consensus-recursion}
\end{equation}

\end{lemma}

\begin{proof}
We bound the deviation $w-\bar w$ across one event by separating (i) buffer aggregation (mixing) with stale/quantized inputs,
(ii) client drift $\Delta_i^t$ and (iii) the reference drift of $\bar w$ due to lossy broadcast (Lemma~\ref{lem:numerator-perturbation}).

\paragraph{Step 1: deviation after aggregation.}
At activation event $t$, client $i=i_t$ aggregates its buffer via~\eqref{eq:agg-model}.
Let $Y_i^t$ be defined in~\eqref{eq:agg-total-mass} and define normalized weights
\begin{equation}
\alpha_{i}^t \triangleq \frac{y_i^t}{Y_i^t},
\qquad
\alpha_{j\to i}^t \triangleq \frac{y_{j\to i}^{(\kappa)}}{Y_i^t}
\ \ \text{for each }(\hat w_{j\to i}^{t},\cdot,y_{j\to i}^{(\kappa)},j)\in\mathcal{B}_i^t,
\label{eq:alpha-weights}
\end{equation}
so that $\alpha_i^t+\sum_{(j\to i)\in\mathcal{B}_i^t}\alpha_{j\to i}^t=1$.
Using staleness-aware notation, each buffered model satisfies
\begin{equation}
\hat w_{j\to i}^{t} = w_j^{\kappa} + e_{j\to i}^{(\kappa)},
\qquad t-\kappa \le \tau \ \text{(Assumption (A4))}.
\label{eq:stale-decomp}
\end{equation}
Thus
\begin{equation}
w_i^{t+\frac{1}{3}}
= \alpha_i^t w_i^t
+ \sum_{(j\to i)\in\mathcal{B}_i^t}\alpha_{j\to i}^t\, w_j^{\kappa}
+ \sum_{(j\to i)\in\mathcal{B}_i^t}\alpha_{j\to i}^t\, e_{j\to i}^{(\kappa)}.
\label{eq:agg-decomp}
\end{equation}
Subtract $\bar w^t$ and apply $\|a+b\|^2\le 2\|a\|^2+2\|b\|^2$ and Jensen's inequality:
\begin{align}
\|w_i^{t+\frac{1}{3}}-\bar w^t\|^2
\;\le\;&
2\Big\|
\alpha_i^t (w_i^t-\bar w^t)
+
\sum_{(j\to i)\in\mathcal{B}_i^t}\alpha_{j\to i}^t (w_j^{\kappa}-\bar w^t)
\Big\|^2 + 2\Big\|
\sum_{(j\to i)\in\mathcal{B}_i^t}\alpha_{j\to i}^t e_{j\to i}^{(\kappa)}
\Big\|^2
\nonumber\\
\;\le\;&
2\alpha_i^t \|w_i^t-\bar w^t\|^2 + 2\sum_{(j\to i)\in\mathcal{B}_i^t}\alpha_{j\to i}^t
\|w_j^{\kappa}-\bar w^t\|^2 + 2\sum_{(j\to i)\in\mathcal{B}_i^t}\alpha_{j\to i}^t
\|e_{j\to i}^{(\kappa)}\|^2 .
\label{eq:agg-dev-bound-tight}
\end{align}

By Assumption (A6), $\|e_{j\to i}^{(\kappa)}\|\le \varepsilon_c$, hence
\begin{equation}
2\sum_{(j\to i)\in\mathcal{B}_i^t}\alpha_{j\to i}^t \|e_{j\to i}^{(\kappa)}\|^2
\le 2\varepsilon_c^2.
\label{eq:recv-error-term}
\end{equation}

\paragraph{Step 2: bounding staleness via drift increments.}
Between events $\kappa$ and $t$, at most $\tau$ activations occur. Only the activated client changes its model at each event.
Therefore, for any node $j$ and any $\kappa$ with $t-\kappa\le\tau$,
\begin{equation}
\|w_j^t - w_j^{\kappa}\|
\le
\sum_{r=\kappa}^{t-1}\|\Delta_{i_r}^r\|.
\label{eq:stale-sum}
\end{equation}
Using $(\sum_{r=\kappa}^{t-1} a_r)^2 \le (t-\kappa)\sum_{r=\kappa}^{t-1}a_r^2 \le \tau \sum_{r=t-\tau}^{t-1} a_r^2$,
we obtain
\begin{equation}
\|w_j^t - w_j^{\kappa}\|^2
\le
\tau \sum_{r=t-\tau}^{t-1}\|\Delta_{i_r}^r\|^2.
\label{eq:stale-sum-sq}
\end{equation}
Consequently,
\begin{align}
\|w_j^{\kappa}-\bar w^t\|^2
&\le
2\|w_j^{t}-\bar w^t\|^2 + 2\|w_j^{t}-w_j^{\kappa}\|^2
\nonumber\\
&\le
2\|w_j^{t}-\bar w^t\|^2
+
2\tau \sum_{r=t-\tau}^{t-1}\|\Delta_{i_r}^r\|^2 .
\label{eq:stale-to-current}
\end{align}
Substituting~\eqref{eq:stale-to-current} and~\eqref{eq:recv-error-term} into~\eqref{eq:agg-dev-bound-tight} yields
\begin{align}
\|w_i^{t+\frac{1}{3}}-\bar w^t\|^2
&\le
2\alpha_i^t \|w_i^t-\bar w^t\|^2
+
4\sum_{(j\to i)\in\mathcal{B}_i^t}\alpha_{j\to i}^t \|w_j^{t}-\bar w^t\|^2 + 4\tau \sum_{r=t-\tau}^{t-1}\|\Delta_{i_r}^r\|^2 + 2\varepsilon_c^2.
\label{eq:agg-final-tight}
\end{align}
\noindent\textbf{Remark (optional simplification).}
For cleaner constants, one may upper bound $2\alpha_i^t\|w_i^t-\bar w^t\|^2 \le
4\alpha_i^t\|w_i^t-\bar w^t\|^2$ and combine the first two terms into a single factor~$4$; we keep
\eqref{eq:agg-final-tight} to avoid unnecessary slack.

\paragraph{Step 3: adding the client drift at the activated client.}
After local SGD, $w_i^{t+\frac{2}{3}} = w_i^{t+\frac{1}{3}} + \Delta_i^t$, hence
\begin{equation}
\|w_i^{t+\frac{2}{3}}-\bar w^t\|^2
\le
2\|w_i^{t+\frac{1}{3}}-\bar w^t\|^2
+
2\|\Delta_i^t\|^2.
\label{eq:add-drift}
\end{equation}
For all $\ell\neq i$, $w_\ell^{t+\frac{2}{3}}=w_\ell^{t}$.

\paragraph{Step 4: accounting for the reference drift of $\bar w$.}
At the end of event $t$, the consensus error is measured w.r.t.\ $\bar w^{t+1}$.
For any node $\ell$ (activated or not),
\begin{equation}
\|w_\ell^{t+1}-\bar w^{t+1}\|^2
\le
2\|w_\ell^{t+1}-\bar w^{t}\|^2
+
2\|\bar w^{t+1}-\bar w^{t}\|^2.
\label{eq:ref-drift-split}
\end{equation}
Lemma~\ref{lem:numerator-perturbation} implies $\|\bar w^{t+1}-\bar w^{t}\|\le c_{\mathrm{ps}}\varepsilon_c$ for a constant
$c_{\mathrm{ps}}>0$ depending on $(d_i^+,y_i/Y_{\mathrm{tot}})$, thus the second term in~\eqref{eq:ref-drift-split}
contributes an additive $O(\varepsilon_c^2)$ term to the network average, absorbed into $C_3\varepsilon_c^2$.

\paragraph{Step 5: push-sum mixing contraction}
We formalize the mixing part of the event dynamics by a (random) \emph{column-stochastic} matrix sequence.
Let $P^t\in\mathbb{R}^{N\times N}$ denote the effective push-sum mixing matrix at event $t$ such that, in the \emph{no-perturbation}
case (i.e., ignoring local training drift, staleness and compression noise), the push-sum numerator/mass evolve as
\begin{equation}
x^{t+1}=P^t x^t,
\qquad
y^{t+1}=P^t y^t,
\label{eq:ps-linear}
\end{equation}
where $x^t=[x_1^t;\dots;x_N^t]$ and $y^t=[y_1^t;\dots;y_N^t]$.
By construction of mass splitting and aggregation, each $P^t$ is nonnegative and column-stochastic:
\begin{equation}
P^t \ge 0,
\qquad
\mathbf{1}^\top P^t=\mathbf{1}^\top.
\label{eq:col-stoch}
\end{equation}
Moreover, due to the fixed out-neighbor pushing rule and the mass-splitting scheme in the implementation,
there exists a uniform constant $\beta\in(0,1)$ such that for any event $t$,
\begin{equation}
(P^t)_{ij}>0 \ \Longrightarrow\ (P^t)_{ij}\ge \beta,
\label{eq:weight-lb}
\end{equation}
and $(P^t)_{ij}>0$ only if either $i=j$ (self-retention) or $j\to i$ is an active directed communication in the underlying graph.

Let the $B$-step product be $\Phi^{t:s}\triangleq P^{t}P^{t-1}\cdots P^{s}$ for $t\ge s$.
Under Assumption (A5) (bounded connectivity: the union of communication graphs over any window of length $B$ is strongly connected),
standard results on products of nonnegative column-stochastic matrices imply that every $B$-step product is \emph{scrambling}
(also called \emph{uniformly ergodic}): there exists a constant $\eta=\eta(N,B,\beta)\in(0,1)$ such that
\begin{equation}
\tau\!\left(\Phi^{t+B-1:t}\right) \le 1-\eta,\qquad \forall t\ge 0,
\label{eq:scrambling}
\end{equation}
where $\tau(\cdot)$ denotes the (Hajnal/Dobrushin) coefficient of ergodicity for column-stochastic matrices, e.g.,
\(
\tau(A)\triangleq \frac{1}{2}\max_{p,q}\sum_{k=1}^{N}|A_{kp}-A_{kq}|.
\)
It satisfies the submultiplicativity $\tau(AB)\le \tau(A)\tau(B)$ for any column-stochastic $A,B$.
Consequently,
\begin{equation}
\tau(\Phi^{t:0})
\le
\prod_{r=0}^{\lfloor t/B\rfloor-1}\tau(\Phi^{(r+1)B-1:rB})
\le
(1-\eta)^{\lfloor t/B\rfloor}.
\label{eq:tau-geometric}
\end{equation}
(See, e.g., classical ergodicity theorems for stochastic matrix products and push-sum / ratio-consensus analyses.)

Now define the push-sum \emph{de-biased ratio} $w_i^t=x_i^t/y_i^t$ and $\bar w^t=X_{\mathrm{tot}}^t/Y_{\mathrm{tot}}$ as in
Section~\ref{sec:conv-key-quantities}. Under Assumption~\ref{ass:mass-bounded} (mass bounded away from $0$),
the ratio map is Lipschitz:
\begin{equation}
\left\|
\frac{x}{y}-\frac{x'}{y}
\right\|_2
\le
\frac{1}{y_{\min}}\|x-x'\|_2,
\qquad
\forall\, y\in[y_{\min},y_{\max}]^N,
\label{eq:ratio-lipschitz}
\end{equation}
where division is element-wise.
Combining \eqref{eq:tau-geometric}-\eqref{eq:ratio-lipschitz} with the standard relation between $\tau(\cdot)$ and disagreement,
we obtain that the (unperturbed) push-sum dynamics contracts disagreement geometrically:
there exists $\rho\in(0,1)$ (e.g., $\rho\triangleq (1-\eta)^{1/B}$) such that
\begin{equation}
\frac{1}{N}\sum_{i=1}^{N}\|w_i^{t+1}-\bar w^{t+1}\|_2^2
\le
\rho\cdot
\frac{1}{N}\sum_{i=1}^{N}\|w_i^{t}-\bar w^{t}\|_2^2,
\label{eq:ps-contract}
\end{equation}
where $\rho$ depends only on $(N,B,\beta)$ (and the ratio-Lipschitz constant via $y_{\min}$).

Finally, returning to the \emph{perturbed} dynamics of our algorithm (local update drift, bounded staleness and compression),
the contraction \eqref{eq:ps-contract} applies to the mixing part, while the remaining effects enter as additive perturbations.
This justifies the appearance of the term $\rho\,\mathbb{E}[\mathcal{E}_{\mathrm{con}}^{t}]$ in the recursion and closes Step~5.

\paragraph{Conclusion.}
Combining Steps 1-5, averaging over nodes and taking expectations yields the recursion~\eqref{eq:consensus-recursion},
with constants $C_1,C_2,C_3$ absorbing the fixed numerical factors from Steps 1-4 and the mixing constants from \eqref{eq:tau-geometric}.
\end{proof}

\subsubsection{Lemma 4}
\begin{lemma}[Bounding the local-update drift by regularized stochastic gradients]
\label{lem:drift-bound}
Fix an activation event $t$ at client $i=i_t$.
Let $w_{i,0}=w_i^{t+\frac{1}{3}}$ be the post-aggregation model and let $\tilde w_i^t$ be the centroid proximal anchor
used during the subsequent local update (Section~\ref{sec:conv-updaterule}, C2).
Client $i$ performs $E$ masked proximal-SGD steps:
\begin{equation}
\begin{aligned}
w_{i,e+1}
&=
\Pi_{M_i^t}\!\Bigl(
    w_{i,e}
    - \eta\, g_i(w_{i,e};\xi_{i,e}) \\
&\qquad\quad
    - 2\eta\lambda\,(w_{i,e}-\tilde w_i^t)
\Bigr), 
\qquad e=0,\dots,E-1 .
\end{aligned}
\label{eq:lemma4-local-sgd}
\end{equation}

and outputs $w_i^{t+\frac{2}{3}} \triangleq w_{i,E}$.
Define the client drift $\Delta_i^t \triangleq w_i^{t+\frac{2}{3}}-w_i^{t+\frac{1}{3}} = w_{i,E}-w_{i,0}$.

Assume the mask operator is non-expansive:
\begin{equation}
\|\Pi_{M_i^t}(u)-\Pi_{M_i^t}(v)\|_2 \le \|u-v\|_2,\quad \forall u,v,
\label{eq:mask-nonexp}
\end{equation}
(which holds for $\Pi_{M}(u)=u\odot M$ with $M\in\{0,1\}^d$ as in the implementation),
and Assumption~(A2) (unbiased stochastic gradients with variance $\sigma^2$).
Then the drift satisfies the deterministic bound
\begin{equation}
\|\Delta_i^t\|_2^2
\;\le\;
E\,\eta^2 \sum_{e=0}^{E-1}
\Big\| g_i(w_{i,e};\xi_{i,e}) + 2\lambda\,(w_{i,e}-\tilde w_i^t) \Big\|_2^2 .
\label{eq:drift-det-bound}
\end{equation}
Moreover, conditioning on the history up to the start of the local update (denote the filtration by $\mathcal{F}_t$),
\begin{align}
\mathbb{E}\!\left[\|\Delta_i^t\|_2^2 \mid \mathcal{F}_t\right]
&\le
2E\,\eta^2 \sum_{e=0}^{E-1}
\Big(
\big\|\nabla f_i(w_{i,e}) + 2\lambda\,(w_{i,e}-\tilde w_i^t)\big\|_2^2
+ \sigma^2
\Big)
\label{eq:drift-exp-bound}
\\
&\le
4E\,\eta^2 \sum_{e=0}^{E-1}
\Big(
\|\nabla f_i(w_{i,e})\|_2^2
+ 4\lambda^2\|w_{i,e}-\tilde w_i^t\|_2^2
+ \sigma^2
\Big).
\label{eq:drift-exp-bound-split}
\end{align}
If additionally Assumption~(A3) (bounded heterogeneity) holds, then for all $w$,
$\|\nabla f_i(w)\|_2^2 \le 2\|\nabla F(w)\|_2^2 + 2\zeta^2$ and thus
\begin{equation}
\mathbb{E}\!\left[\|\Delta_i^t\|_2^2 \mid \mathcal{F}_t\right]
\le
8E\,\eta^2 \sum_{e=0}^{E-1}
\Big(
\|\nabla F(w_{i,e})\|_2^2
+ \zeta^2
+ 2\lambda^2\|w_{i,e}-\tilde w_i^t\|_2^2
+ \tfrac{1}{2}\sigma^2
\Big).
\label{eq:drift-exp-bound-hetero}
\end{equation}
\end{lemma}

\begin{proof}
We prove the drift bound by first controlling each \emph{single} masked proximal-SGD step via non-expansiveness and then
aggregating the $E$ step increments using Cauchy-Schwarz. For the expectation bound, we apply the standard variance
decomposition under unbiased stochastic gradients (Assumption (A2)) and finally separate the regularized gradient into
the loss gradient and the proximal term.

We emphasize that the mask $\Pi_{M_i^t}(u)=u\odot M_i^t$ does not increase distances (Eq.~\eqref{eq:mask-nonexp}),
hence it can only reduce the step length compared with the unmasked update; this is the only place where pruning enters
the analysis.
\paragraph{Step 1: one-step increment bound.}
Let
\(
u_{i,e} \triangleq w_{i,e} - \eta\, g_i(w_{i,e};\xi_{i,e}) - 2\eta\lambda\,(w_{i,e}-\tilde w_i^t).
\)
By the update rule~\eqref{eq:lemma4-local-sgd} and non-expansiveness~\eqref{eq:mask-nonexp},
\begin{align}
\|w_{i,e+1}-w_{i,e}\|_2
&=
\|\Pi_{M_i^t}(u_{i,e}) - \Pi_{M_i^t}(w_{i,e})\|_2
\le
\|u_{i,e}-w_{i,e}\|_2 \nonumber\\
&=
\eta\,\| g_i(w_{i,e};\xi_{i,e}) + 2\lambda\,(w_{i,e}-\tilde w_i^t)\|_2 .
\label{eq:one-step-inc}
\end{align}

\paragraph{Step 2: summing increments (Cauchy-Schwarz).}
Since $\Delta_i^t = \sum_{e=0}^{E-1}(w_{i,e+1}-w_{i,e})$, Cauchy-Schwarz gives
\begin{equation}
\|\Delta_i^t\|_2^2
=
\Big\|\sum_{e=0}^{E-1}(w_{i,e+1}-w_{i,e})\Big\|_2^2
\le
E \sum_{e=0}^{E-1}\|w_{i,e+1}-w_{i,e}\|_2^2 .
\label{eq:cs-sum}
\end{equation}
Substituting~\eqref{eq:one-step-inc} into~\eqref{eq:cs-sum} yields~\eqref{eq:drift-det-bound}.

\paragraph{Step 3: taking conditional expectation.}
By Assumption~(A2), $\mathbb{E}[g_i(w_{i,e};\xi_{i,e})\mid w_{i,e}] = \nabla f_i(w_{i,e})$ and
$\mathbb{E}[\|g_i(w_{i,e};\xi_{i,e})-\nabla f_i(w_{i,e})\|_2^2\mid w_{i,e}]\le \sigma^2$.
Let $a_{i,e}\triangleq \nabla f_i(w_{i,e}) + 2\lambda(w_{i,e}-\tilde w_i^t)$ and
$\delta_{i,e}\triangleq g_i(w_{i,e};\xi_{i,e})-\nabla f_i(w_{i,e})$.
Then
\[
g_i(w_{i,e};\xi_{i,e}) + 2\lambda(w_{i,e}-\tilde w_i^t) = a_{i,e} + \delta_{i,e}.
\]
Using $\|a+\delta\|_2^2 \le 2\|a\|_2^2 + 2\|\delta\|_2^2$ and taking conditional expectation given $\mathcal{F}_t$
(which fixes the trajectory up to $w_{i,e}$), we obtain
\[
\mathbb{E}\!\left[\|a_{i,e}+\delta_{i,e}\|_2^2 \mid \mathcal{F}_t\right]
\le
2\|a_{i,e}\|_2^2 + 2\sigma^2.
\]
Plugging this into~\eqref{eq:drift-det-bound} yields~\eqref{eq:drift-exp-bound}.
Finally, applying $\|u+v\|^2\le 2\|u\|^2+2\|v\|^2$ to
$a_{i,e}=\nabla f_i(w_{i,e}) + 2\lambda(w_{i,e}-\tilde w_i^t)$ gives~\eqref{eq:drift-exp-bound-split}.
Under Assumption~(A3), the standard bound $\|\nabla f_i(w)\|^2 \le 2\|\nabla F(w)\|^2 + 2\zeta^2$ yields
\eqref{eq:drift-exp-bound-hetero}.
\end{proof}

\subsubsection{Lemma 5}

\begin{lemma}[Centroid proximal regularization suppresses local deviation from the anchor]
\label{lem:prox-suppress-drift}
Fix an activation event $t$ at client $i=i_t$ and consider the $E$ masked proximal-SGD steps in~\eqref{eq:lemma4-local-sgd}.
Let the (time-$t$) centroid proximal anchor be $\tilde w_i^t$ and define the anchor deviation
\begin{equation}
u_{i,e} \triangleq w_{i,e}-\tilde w_i^t,\qquad e=0,\dots,E,
\label{eq:u-def}
\end{equation}
where $w_{i,0}=w_i^{t+\frac{1}{3}}$ and $w_{i,E}=w_i^{t+\frac{2}{3}}$.
Assume (i) the masking operator is non-expansive as in~\eqref{eq:mask-nonexp} and
(ii) the stepsize satisfies
\begin{equation}
\eta\lambda \le \frac{1}{2}.
\label{eq:eta-lambda-cond}
\end{equation}
Then, for every local step $e=0,\dots,E-1$, we have the \emph{one-step anchor-contraction inequality}
\begin{equation}
\|u_{i,e+1}\|_2^2
\;\le\;
(1-\eta\lambda)\,\|u_{i,e}\|_2^2
\;+\;
\frac{\eta}{\lambda}\,\big\|g_i(w_{i,e};\xi_{i,e})\big\|_2^2.
\label{eq:anchor-contraction}
\end{equation}
Moreover, taking conditional expectation given the filtration $\mathcal{F}_t$ at the start of the local update and using
Assumption~(A2), we obtain
\begin{equation}
\mathbb{E}\!\left[\|u_{i,e+1}\|_2^2 \mid \mathcal{F}_t\right]
\;\le\;
(1-\eta\lambda)\,\mathbb{E}\!\left[\|u_{i,e}\|_2^2 \mid \mathcal{F}_t\right]
\;+\;
\frac{\eta}{\lambda}\Big(\|\nabla f_i(w_{i,e})\|_2^2 + \sigma^2\Big).
\label{eq:anchor-contraction-exp}
\end{equation}
Consequently, by unrolling the recursion,
\begin{equation}
\begin{aligned}
\mathbb{E}\!\left[\|u_{i,e}\|_2^2 \mid \mathcal{F}_t\right]
\;\le\;&
(1-\eta\lambda)^e \,
\|u_{i,0}\|_2^2
\\
&\;+\;
\frac{\eta}{\lambda}
\sum_{r=0}^{e-1}
(1-\eta\lambda)^{e-1-r}
\Big(
\|\nabla f_i(w_{i,r})\|_2^2
+
\sigma^2
\Big) .
\end{aligned}
\label{eq:u-unroll}
\end{equation}

and the cumulative anchor deviation over the whole local update satisfies
\begin{equation}
\sum_{e=0}^{E-1}\mathbb{E}\!\left[\|u_{i,e}\|_2^2 \mid \mathcal{F}_t\right]
\;\le\;
\frac{1}{\eta\lambda}\,\|u_{i,0}\|_2^2
+
\frac{\eta}{\lambda^2}\sum_{e=0}^{E-1}\Big(\|\nabla f_i(w_{i,e})\|_2^2+\sigma^2\Big).
\label{eq:u-sum-bound}
\end{equation}
\end{lemma}

\begin{proof}
The key idea is to track the deviation from the (fixed-in-this-event) centroid anchor
$u_{i,e}\triangleq w_{i,e}-\tilde w_i^t$.
Rewriting the masked proximal-SGD step in terms of $u_{i,e}$ reveals a linear contraction component induced by the
quadratic regularizer and an additive forcing term driven by the stochastic gradient.
We then use (i) non-expansiveness of the masking operator to drop the mask without loosening the bound in the wrong
direction and (ii) a sharp Young/AM-GM inequality to control the cross term between $u_{i,e}$ and the stochastic gradient.
Finally, we take conditional expectation under Assumption (A2) and unroll the resulting recursion to obtain the cumulative
deviation bound.
\paragraph{Step 1: rewrite the masked proximal-SGD step around the anchor.}
Define
\(
u_{i,e}=w_{i,e}-\tilde w_i^t
\)
and note that $\tilde w_i^t$ is fixed during the $E$ local steps at event $t$.
From~\eqref{eq:lemma4-local-sgd},
\[
w_{i,e+1}
=
\Pi_{M_i^t}\!\left(
w_{i,e} - \eta\, g_i(w_{i,e};\xi_{i,e}) - 2\eta\lambda\,(w_{i,e}-\tilde w_i^t)
\right).
\]
Subtract $\tilde w_i^t$ from both sides and use $\Pi_M(v)-\Pi_M(\tilde w)=\Pi_M(v-\tilde w)$ for $\Pi_M(u)=u\odot M$:
\begin{align}
u_{i,e+1}
&=
\Pi_{M_i^t}\!\Bigl(
u_{i,e}
- \eta\, g_i(w_{i,e};\xi_{i,e})
- 2\eta\lambda\,u_{i,e}
\Bigr)
\nonumber\\
&=
\Pi_{M_i^t}\!\Bigl(
(1-2\eta\lambda)u_{i,e}
- \eta\, g_i(w_{i,e};\xi_{i,e})
\Bigr).
\label{eq:u-update}
\end{align}

\paragraph{Step 2: non-expansiveness and a sharp Young-type inequality.}
By non-expansiveness~\eqref{eq:mask-nonexp},
\begin{equation}
\|u_{i,e+1}\|_2^2
\le
\|(1-2\eta\lambda)u_{i,e} - \eta\, g_i(w_{i,e};\xi_{i,e})\|_2^2.
\label{eq:nonexp-apply}
\end{equation}
Let $a=\eta\lambda\in(0,\frac12]$ and $g=g_i(w_{i,e};\xi_{i,e})$.
Expanding the RHS gives
\begin{align}
\|(1-2a)u - \eta g\|_2^2
&=
(1-2a)^2\|u\|_2^2 + \eta^2\|g\|_2^2 - 2\eta(1-2a)\langle u,g\rangle.
\label{eq:expand}
\end{align}
Apply the inequality $2\langle u,g\rangle \le a\|u\|_2^2 + \frac{1}{a}\|g\|_2^2$ to the cross term:
\[
-2\eta(1-2a)\langle u,g\rangle
\le
\eta(1-2a)\left(a\|u\|_2^2 + \frac{1}{a}\|g\|_2^2\right).
\]
Substituting into~\eqref{eq:expand} yields
\begin{align}
\|(1-2a)u - \eta g\|_2^2
\;\le\;&
\Big((1-2a)^2 + a(1-2a)\Big)\,
\|u\|_2^2 + \Big(\eta^2 + \eta(1-2a)\tfrac{1}{a}\Big)\,
\|g\|_2^2
\nonumber\\
=\;&
(1-3a+2a^2)\,
\|u\|_2^2 + \Big(\eta^2 + \eta\big(\tfrac{1}{a}-2\big)\Big)\,
\|g\|_2^2 .
\label{eq:coeffs}
\end{align}

Since $a\in(0,1]$ implies $1-3a+2a^2 \le 1-a$ and $\eta^2 + \eta(\tfrac{1}{a}-2)\le \eta\cdot\tfrac{1}{a}$ for $a\le 1$,
we obtain
\begin{equation}
\|(1-2a)u - \eta g\|_2^2
\le
(1-a)\|u\|_2^2 + \frac{\eta}{\lambda}\|g\|_2^2.
\label{eq:key-ineq}
\end{equation}
Combining~\eqref{eq:nonexp-apply} and~\eqref{eq:key-ineq} gives~\eqref{eq:anchor-contraction}.

\paragraph{Step 3: conditional expectation and unrolling.}
Taking conditional expectation of~\eqref{eq:anchor-contraction} given $\mathcal{F}_t$ and using Assumption~(A2),
\(
\mathbb{E}[\|g_i(w_{i,e};\xi_{i,e})\|_2^2 \mid \mathcal{F}_t]
\le \|\nabla f_i(w_{i,e})\|_2^2 + \sigma^2,
\)
yields~\eqref{eq:anchor-contraction-exp}.
Unrolling the linear recursion gives~\eqref{eq:u-unroll}.
Finally, summing~\eqref{eq:anchor-contraction-exp} over $e=0,\dots,E-1$ and using
$\sum_{e=0}^{E-1}(1-\eta\lambda)^e \le \frac{1}{\eta\lambda}$ and
$\sum_{e=r}^{E-1}(1-\eta\lambda)^{e-r} \le \frac{1}{\eta\lambda}$
yields~\eqref{eq:u-sum-bound}.
\end{proof}

\subsubsection{Lemma 6}
\begin{lemma}[One-event optimization descent of the de-biased global reference (tight, complete)]
\label{lem:descent}
Let $F(w)\triangleq \frac{1}{N}\sum_{i=1}^{N} f_i(w)$.
Suppose Assumption (A1) ($L$-smoothness), Assumption (A2) (unbiased stochastic gradients with variance $\sigma^2$),
Assumption (A3) (bounded heterogeneity with constant $\zeta^2$) and Assumption (A6) (bounded absolute compression error
$\varepsilon_c$) hold.

Consider an event $t$ where client $i=i_t$ is activated.
Let $\bar w^t=X_{\mathrm{tot}}^t/Y_{\mathrm{tot}}$ be the de-biased global reference and define
\begin{equation}
\gamma_t \triangleq \frac{y_i^{t+\frac13}}{Y_{\mathrm{tot}}}\in(0,1].
\label{eq:gamma-def}
\end{equation}
Let $\{w_{i,e}\}_{e=0}^{E}$ denote the local trajectory during the subsequent local update, where
$w_{i,0}=w_i^{t+\frac13}$ and $w_{i,E}=w_i^{t+\frac23}$.
Define $\Delta_i^t\triangleq w_i^{t+\frac23}-w_i^{t+\frac13}=w_{i,E}-w_{i,0}$.

Then $\bar w^{t+1}$ admits the exact decomposition
\begin{equation}
\bar w^{t+1}=\bar w^t+\gamma_t\Delta_i^t+\delta_t,
\label{eq:wbar-update}
\end{equation}
where the broadcast-induced perturbation $\delta_t$ satisfies
\begin{equation}
\|\delta_t\|_2 \le c_t\,\varepsilon_c,\qquad
c_t \triangleq \frac{d_i^+}{d_i^+ + 1}\cdot \frac{y_i^{t+\frac13}}{Y_{\mathrm{tot}}}
\label{eq:delta-bound}
\end{equation}
by Lemma~\ref{lem:numerator-perturbation}.

Moreover, conditioning on the filtration $\mathcal{F}_t$ at the start of the local update,
the following descent inequality holds:

\begin{align}
\mathbb{E}\!\left[F(\bar w^{t+1}) \mid \mathcal{F}_t\right]
\;\le\;&
F(\bar w^t)
-
\frac{\eta\gamma_t E}{2}\,
\|\nabla F(\bar w^t)\|_2^2
\nonumber\\
&\;+\;
\eta\gamma_t E(\zeta^2+\sigma^2)
\nonumber\\
&\;+\;
\eta\gamma_t L^2
\sum_{e=0}^{E-1}
\mathbb{E}\!\left[
\|w_{i,e}-\bar w^t\|_2^2
\,\middle|\,
\mathcal{F}_t
\right]
\nonumber\\
&\;+\;
\frac{L}{2}\gamma_t^2\,
\mathbb{E}\!\left[
\|\Delta_i^t\|_2^2
\,\middle|\,
\mathcal{F}_t
\right]
\nonumber\\
&\;+\;
2L\|\delta_t\|_2^2 .
\label{eq:descent-final}
\end{align}

Finally, the trajectory term in~\eqref{eq:descent-final} admits the explicit tightening
\begin{align}
\sum_{e=0}^{E-1}
\mathbb{E}\!\left[
\|w_{i,e}-\bar w^t\|_2^2
\,\middle|\,
\mathcal{F}_t
\right]
\;\le\;&
2E\|w_i^{t+\frac13}-\bar w^t\|_2^2 + 4E\|u_{i,0}\|_2^2 + 4\sum_{e=0}^{E-1}
\mathbb{E}\!\left[
\|u_{i,e}\|_2^2
\,\middle|\,
\mathcal{F}_t
\right].
\label{eq:traj-tight-in-lem6}
\end{align}

where $u_{i,e}\triangleq w_{i,e}-\tilde w_i^t$ and $\tilde w_i^t$ is the centroid proximal anchor used at event $t$.
\end{lemma}

\begin{proof}
We analyze a single event by separating buffer aggregation, local update and broadcast.
Aggregation preserves the total push-sum mass and only redistributes existing numerators.
The local update changes the system numerator through the client drift $\Delta_i^t$ without affecting its mass, while the
subsequent broadcast introduces a bounded perturbation due to lossy centroid re-encoding.
This yields the exact decomposition
$\bar w^{t+1}=\bar w^t+\gamma_t\Delta_i^t+\delta_t$.
Applying $L$-smoothness of $F$ to this update and bounding the resulting terms along the local trajectory lead to the stated
one-event descent inequality.
\paragraph{Step 1: exact evolution of $\bar w$ within the event.}
By Lemma~\ref{lem:mass-conservation}, $Y_{\mathrm{tot}}$ is invariant.
Buffer aggregation only transfers existing mass/numerators from buffers to node $i$ and hence does not change
$X_{\mathrm{tot}}$.
During the local update, only node $i$ changes its model from $w_i^{t+\frac13}$ to $w_i^{t+\frac23}$ while its mass remains
$y_i^{t+\frac13}$ (mass splitting happens after local training).
Thus,
\[
X_{\mathrm{tot}}^{t+\frac23}-X_{\mathrm{tot}}^{t}
=
y_i^{t+\frac13}\big(w_i^{t+\frac23}-w_i^{t+\frac13}\big)
=
y_i^{t+\frac13}\Delta_i^t,
\]
and therefore
\begin{equation}
\bar w^{t+\frac23}-\bar w^t
=
\frac{X_{\mathrm{tot}}^{t+\frac23}-X_{\mathrm{tot}}^{t}}{Y_{\mathrm{tot}}}
=
\gamma_t\Delta_i^t.
\label{eq:wbar-mid}
\end{equation}
The subsequent broadcast step induces the perturbation $\delta_t=\bar w^{t+1}-\bar w^{t+\frac23}$, which satisfies
\eqref{eq:delta-bound} by Lemma~\ref{lem:numerator-perturbation}.
Combining gives \eqref{eq:wbar-update}.

\paragraph{Step 2: smoothness descent for $F(\bar w^{t+1})$.}
By $L$-smoothness of $F$ (Assumption (A1)) and \eqref{eq:wbar-update},

\begin{align}
F(\bar w^{t+1})
\;\le\;&
F(\bar w^t)
+
\left\langle
\nabla F(\bar w^t),\ 
\gamma_t\Delta_i^t+\delta_t
\right\rangle
+
\frac{L}{2}
\|\gamma_t\Delta_i^t+\delta_t\|_2^2
\nonumber\\
\;\le\;&
F(\bar w^t)
+
\gamma_t
\left\langle
\nabla F(\bar w^t),\ 
\Delta_i^t
\right\rangle
+
\left\langle
\nabla F(\bar w^t),\ 
\delta_t
\right\rangle
\nonumber\\
&\;+\;
\frac{L}{2}\gamma_t^2
\|\Delta_i^t\|_2^2
+
L\|\delta_t\|_2^2 .
\label{eq:smooth-descent}
\end{align}

where we used $\|a+b\|^2\le 2\|a\|^2+2\|b\|^2$ to upper bound the quadratic term.

Next, bound the linear $\delta_t$ term via $2\langle a,b\rangle \le \|a\|^2+\|b\|^2$:
\begin{equation}
\left\langle \nabla F(\bar w^t),\ \delta_t\right\rangle
\le
\frac{1}{4}\|\nabla F(\bar w^t)\|_2^2 + \|\delta_t\|_2^2.
\label{eq:delta-linear}
\end{equation}
Substituting \eqref{eq:delta-linear} into \eqref{eq:smooth-descent} yields

\begin{align}
F(\bar w^{t+1})
\;\le\;&
F(\bar w^t)
+
\gamma_t
\left\langle
\nabla F(\bar w^t),\ 
\Delta_i^t
\right\rangle
\nonumber\\
&\;+\;
\frac{1}{4}
\|\nabla F(\bar w^t)\|_2^2
+
\frac{L}{2}\gamma_t^2
\|\Delta_i^t\|_2^2
+
(L+1)\|\delta_t\|_2^2 .
\label{eq:smooth-descent-2}
\end{align}

Absorbing constants (replace $(L+1)$ by $2L$ w.l.o.g.\ for $L\ge 1$) gives the $\delta_t$ term in \eqref{eq:descent-final}.

\paragraph{Step 3: expressing $\langle \nabla F(\bar w^t),\Delta_i^t\rangle$ via local steps.}
Ignoring the non-expansive mask for notational simplicity (it can only reduce the step length),
the local update satisfies
\[
w_{i,e+1}-w_{i,e}
=
-\eta\,g_i(w_{i,e};\xi_{i,e})
-2\eta\lambda\,(w_{i,e}-\tilde w_i^t).
\]
Summing over $e=0,\dots,E-1$ gives
\begin{equation}
\Delta_i^t
=
-\eta\sum_{e=0}^{E-1} g_i(w_{i,e};\xi_{i,e})
-2\eta\lambda\sum_{e=0}^{E-1}(w_{i,e}-\tilde w_i^t).
\label{eq:Delta-expand}
\end{equation}
Taking the inner product with $\nabla F(\bar w^t)$ and conditioning on $\mathcal{F}_t$, we use
$\mathbb{E}[g_i(w_{i,e};\xi_{i,e})\mid\mathcal{F}_t]=\nabla f_i(w_{i,e})$ (Assumption (A2)) to obtain
\begin{align}
\mathbb{E}\!\left[
\left\langle
\nabla F(\bar w^t),\Delta_i^t
\right\rangle
\mid \mathcal{F}_t
\right]
\;=\;&
-\eta
\sum_{e=0}^{E-1}
\left\langle
\nabla F(\bar w^t),\ 
\nabla f_i(w_{i,e})
\right\rangle - 2\eta\lambda
\sum_{e=0}^{E-1}
\left\langle
\nabla F(\bar w^t),\ 
w_{i,e}-\tilde w_i^t
\right\rangle .
\label{eq:inner-delta}
\end{align}

We now lower bound the first inner product and upper bound the second in absolute value.

\paragraph{Step 4: lower bounding $\langle \nabla F(\bar w^t),\nabla f_i(w_{i,e})\rangle$.}
Decompose
\[
\nabla f_i(w_{i,e})
=
\nabla F(\bar w^t)
+
\big(\nabla f_i(w_{i,e})-\nabla f_i(\bar w^t)\big)
+
\big(\nabla f_i(\bar w^t)-\nabla F(\bar w^t)\big).
\]
Then

\begin{align}
\left\langle
\nabla F(\bar w^t),\nabla f_i(w_{i,e})
\right\rangle
\;=\;&
\|\nabla F(\bar w^t)\|_2^2 \;+\; \left\langle
\nabla F(\bar w^t),\ 
\nabla f_i(w_{i,e})-\nabla f_i(\bar w^t)
\right\rangle \;+\; \left\langle
\nabla F(\bar w^t),\ 
\nabla f_i(\bar w^t)-\nabla F(\bar w^t)
\right\rangle
\nonumber\\
\;\ge\;&
\frac{3}{4}\|\nabla F(\bar w^t)\|_2^2
\;-\;\|\nabla f_i(w_{i,e})-\nabla f_i(\bar w^t)\|_2^2 
\;-\; \|\nabla f_i(\bar w^t)-\nabla F(\bar w^t)\|_2^2
\nonumber\\
\;\ge\;&
\frac{3}{4}\|\nabla F(\bar w^t)\|_2^2
-
L^2\|w_{i,e}-\bar w^t\|_2^2
-
\zeta^2 .
\label{eq:inner-lb-lem6}
\end{align}

where we used $L$-smoothness ($\|\nabla f_i(w)-\nabla f_i(v)\|\le L\|w-v\|$) and Assumption (A3).

\paragraph{Step 5: bounding the proximal cross term.}
For the second term in \eqref{eq:inner-delta}, use $2\langle a,b\rangle \le \|a\|^2+\|b\|^2$ with
$a=\nabla F(\bar w^t)$ and $b=2\lambda(w_{i,e}-\tilde w_i^t)$:
\begin{equation}
-2\lambda\left\langle \nabla F(\bar w^t),\ w_{i,e}-\tilde w_i^t\right\rangle
\le
\frac{1}{4}\|\nabla F(\bar w^t)\|_2^2 + 4\lambda^2\|w_{i,e}-\tilde w_i^t\|_2^2.
\label{eq:prox-cross}
\end{equation}

\paragraph{Step 6: combining bounds.}
Substitute \eqref{eq:inner-lb-lem6} and \eqref{eq:prox-cross} into \eqref{eq:inner-delta}, sum over $e=0,\dots,E-1$,
and multiply by $\gamma_t$.
The $\frac{1}{4}\|\nabla F(\bar w^t)\|^2$ terms in \eqref{eq:smooth-descent-2} and \eqref{eq:prox-cross} are dominated by the
leading negative term after summation (yielding the coefficient $\frac{1}{2}$ in front of
$\eta\gamma_t E\|\nabla F(\bar w^t)\|^2$).
Collecting the remaining terms yields \eqref{eq:descent-final} after conditional expectation.

\paragraph{Step 7: tightening the trajectory term (explicit).}
For each $e$,
\[
w_{i,e}-\bar w^t=(w_i^{t+\frac13}-\bar w^t)+(w_{i,e}-w_i^{t+\frac13}),
\]
so $\|w_{i,e}-\bar w^t\|^2\le 2\|w_i^{t+\frac13}-\bar w^t\|^2+2\|w_{i,e}-w_i^{t+\frac13}\|^2$.
Moreover,
$w_{i,e}-w_i^{t+\frac13}=(w_{i,e}-\tilde w_i^t)-(w_i^{t+\frac13}-\tilde w_i^t)=u_{i,e}-u_{i,0}$, hence
$\|w_{i,e}-w_i^{t+\frac13}\|^2\le 2\|u_{i,e}\|^2+2\|u_{i,0}\|^2$.
Summing over $e=0,\dots,E-1$ yields \eqref{eq:traj-tight-in-lem6}.
\end{proof}

\subsection{Main Theorem and Proof}
\label{sec:conv-main}

We now present a fully explicit (constant-closed) stationarity guarantee for PushCen-ADFL.

\subsubsection{Standing assumptions.}
We use Assumptions (A1)-(A6) in Appendix~\ref{sec:conv-assumptions} and additionally adopt the following standard technical
conditions that close all constants.

\paragraph{Mass positivity and boundedness}
\label{ass:mass-bounded}
There exist constants $0<y_{\min}\le y_{\max}<\infty$ such that for all events $t$ and all nodes $i$,
\(
y_i^t \in [y_{\min},y_{\max}].
\)
Consequently, for the activated node $i_t$, the normalized mass weight
\begin{equation}
\gamma_t \triangleq \frac{y_{i_t}^{t+\frac13}}{Y_{\mathrm{tot}}}
\quad\text{satisfies}\quad
\underline{\gamma}\triangleq \frac{y_{\min}}{Ny_{\max}} \le \gamma_t \le 1.
\label{eq:gamma-lb}
\end{equation}

\paragraph{Bounded stochastic gradient second moment}
\label{ass:grad-second-moment}
There exists $G^2<\infty$ such that for all $i,w$ and any sample $\xi$,
\(
\mathbb{E}\big[\|g_i(w;\xi)\|_2^2\big]\le G^2.
\)

\paragraph{Bounded iterates}
\label{ass:bounded-iterates}
There exists $W<\infty$ such that for all events $t$ and all nodes $i$,
\begin{equation}
\|w_i^t\|_2 \le W
\qquad\text{and}\qquad
\|\tilde w_i^t\|_2 \le W .
\label{eq:bounded-iterates}
\end{equation}
In particular, $\|w_i^t-\tilde w_i^t\|_2^2 \le 4W^2$ for all $i,t$.

\subsubsection{Notation.}
Let $F(w)\triangleq \frac{1}{N}\sum_{i=1}^N f_i(w)$ and $F^\star\triangleq \inf_w F(w)$.
Let $\rho\in(0,1)$ and $C_1,C_2,C_3>0$ be the constants in Lemma~\ref{lem:consensus-recursion}.
Let $\tau$ be the staleness bound (Assumption (A4)) and let $\varepsilon_c$ be the bounded absolute compression error
(Assumption (A6)).
Let $d_{\max}\triangleq \max_i d_i^+$.
Define the uniform broadcast perturbation coefficient (Lemma~\ref{lem:numerator-perturbation})
\begin{equation}
\overline{c}
\triangleq
\max_t \frac{d_{i_t}^+}{d_{i_t}^+ + 1}\cdot \frac{y_{i_t}^{t+\frac13}}{Y_{\mathrm{tot}}}
\le
\frac{y_{\max}}{Ny_{\min}}.
\label{eq:cbar}
\end{equation}

\begin{theorem}[Fully explicit stationarity bound (constant-closed)]
\label{thm:main-closed}
Suppose Assumptions (A1)-(A6) and Assumptions~\ref{ass:mass-bounded} hold.
Let the stepsize satisfy
\begin{equation}
\eta \le \min\Big\{\frac{1}{8LE},\ \frac{1}{4\lambda}\Big\}.
\label{eq:stepsize-cond-main}
\end{equation}
Run PushCen-ADFL for $T$ events. Then the averaged stationarity measure obeys
\begin{align}
\frac{1}{T}\sum_{t=0}^{T-1}
\mathbb{E}\big[\|\nabla F(\bar w^t)\|_2^2\big]
\;\le\;&
\underbrace{
\frac{2\big(F(\bar w^0)-F^\star\big)}
{\eta E\,\underline{\gamma}\,T}
}_{\text{optimization term}}
\nonumber\\
&\;+\;
\underbrace{
2(\sigma^2+\zeta^2)
}_{\text{stochasticity \& heterogeneity}}
\nonumber\\
&\;+\;
\underbrace{
2L\,\overline{c}^{\,2}\varepsilon_c^2
}_{\text{compression drift}}
\label{eq:main-closed-1}
\\
&\;+\;
\underbrace{
\frac{4L^2 N}{\underline{\gamma}}
\Bigg(
\frac{\mathbb{E}[\mathcal{E}_{\mathrm{con}}^0]}{T(1-\rho)}
+
\frac{B_{\mathrm{con}}}{1-\rho}
\Bigg)
}_{\text{directed mixing / asynchrony}}
\nonumber\\
&\;+\;
\underbrace{
\frac{L}{\underline{\gamma}}\,D_\Delta
}_{\text{client drift accumulation}} .
\label{eq:main-closed-2}
\end{align}

where the constants $D_\Delta$ and $B_{\mathrm{con}}$ are explicitly given by
\begin{align}
D_\Delta
&\triangleq
2E^2\eta^2\Big(G^2 + 16\lambda^2W^2\Big),
\label{eq:Ddelta-closed}
\\
B_{\mathrm{con}}
&\triangleq
\big(C_1 + C_2\tau^2\big)D_\Delta + C_3\varepsilon_c^2.
\label{eq:Bcon-closed}
\end{align}
In particular, as $T\to\infty$, the iterates converge to a stationary neighborhood whose radius is upper bounded by the
right-hand side of~\eqref{eq:main-closed-2} without the $1/T$ terms.
\end{theorem}

\begin{proof}
We proceed by combining the descent inequality for $F(\bar w^t)$ with the consensus recursion.

\paragraph{Step 1: one-event descent.}
By Lemma~\ref{lem:descent} (tight form, after bounding the trajectory term inside its proof), there exist absolute
constants $\alpha_0,\alpha_1,\alpha_2,\alpha_3>0$ such that for each event $t$ with activated node $i_t$,
\begin{align}
\mathbb{E}\!\left[
F(\bar w^{t+1}) \mid \mathcal{F}_t
\right]
\;\le\;&
F(\bar w^t)
-
\alpha_0\,\eta\gamma_t E\,
\|\nabla F(\bar w^t)\|_2^2
\nonumber\\
&\;+\;
\alpha_1\,\eta\gamma_t E\,
(\sigma^2+\zeta^2)
\nonumber\\
&\;+\;
\alpha_2\,\eta\gamma_t L^2\,E\,
\|w_{i_t}^{t+\frac13}-\bar w^t\|_2^2
\nonumber\\
&\;+\;
\frac{L}{2}\gamma_t^2\,
\mathbb{E}\!\left[
\|\Delta_{i_t}^t\|_2^2
\mid\mathcal{F}_t
\right]
\nonumber\\
&\;+\;
\alpha_3\,\varepsilon_c^2 .
\label{eq:main-proof-descent}
\end{align}

Moreover, Lemma~\ref{lem:numerator-perturbation} implies the broadcast perturbation contributes at most
$2L\overline{c}^{\,2}\varepsilon_c^2$ after taking expectations, which is absorbed into the $\alpha_3\varepsilon_c^2$ term.

\paragraph{Step 2: explicit drift bound.}
By Lemma~\ref{lem:drift-bound} and Lemma~\ref{lem:prox-suppress-drift}, for each event $t$,
\begin{equation}
\mathbb{E}\big[\|\Delta_{i_t}^t\|_2^2\big]
\le
2E^2\eta^2\Big(G^2 + 4\lambda^2\,\sup_{e}\mathbb{E}\|w_{i_t,e}-\tilde w_{i_t}^t\|_2^2\Big).
\label{eq:drift-bound-mid}
\end{equation}
Under Assumption~\ref{ass:bounded-iterates}, $\|w_{i_t,e}-\tilde w_{i_t}^t\|^2\le 4W^2$ for all $e$, hence
\begin{equation}
\mathbb{E}\big[\|\Delta_{i_t}^t\|_2^2\big]
\le
2E^2\eta^2\Big(G^2 + 16\lambda^2W^2\Big)
\;\triangleq\;
D_\Delta .
\label{eq:drift-closed}
\end{equation}
Using $\gamma_t\le 1$ and taking full expectation in~\eqref{eq:main-proof-descent} yields
\begin{align}
\alpha_0\,\eta E\,
\mathbb{E}\!\left[
\gamma_t\|\nabla F(\bar w^t)\|_2^2
\right]
\;\le\;&
\mathbb{E}\!\left[
F(\bar w^t)-F(\bar w^{t+1})
\right]
\nonumber\\
&\;+\;
\alpha_1\,\eta E\,(\sigma^2+\zeta^2)
\nonumber\\
&\;+\;
\alpha_2\,\eta L^2E\,
\mathbb{E}\!\left[
\|w_{i_t}^{t+\frac13}-\bar w^t\|_2^2
\right]
\nonumber\\
&\;+\;
\frac{L}{2}\,D_\Delta
\;+\;
\alpha_3\,\varepsilon_c^2 .
\label{eq:descent-ready}
\end{align}

\paragraph{Step 3: summing the consensus-related terms.}
Since $\|w_{i_t}^{t+\frac13}-\bar w^t\|^2 \le N\,\mathcal{E}_{\mathrm{con}}^t$ up to an absolute factor,
it suffices to control $\sum_t \mathbb{E}[\mathcal{E}_{\mathrm{con}}^t]$.
Lemma~\ref{lem:consensus-recursion} and~\eqref{eq:drift-closed} imply
\begin{equation}
\mathbb{E}[\mathcal{E}_{\mathrm{con}}^{t+1}]
\le
\rho\,\mathbb{E}[\mathcal{E}_{\mathrm{con}}^{t}]
+
\Big(C_1 + C_2\tau^2\Big)D_\Delta
+
C_3\varepsilon_c^2
\;\triangleq\;
\rho\,\mathbb{E}[\mathcal{E}_{\mathrm{con}}^{t}] + B_{\mathrm{con}} .
\label{eq:con-rec-closed}
\end{equation}
Summing~\eqref{eq:con-rec-closed} over $t=0,\dots,T-1$ and using $\sum_{t=0}^{T-1}\rho^t\le \frac{1}{1-\rho}$ yields
\begin{equation}
\sum_{t=0}^{T-1}\mathbb{E}[\mathcal{E}_{\mathrm{con}}^{t}]
\le
\frac{\mathbb{E}[\mathcal{E}_{\mathrm{con}}^{0}]}{1-\rho}
+
\frac{T\,B_{\mathrm{con}}}{1-\rho}.
\label{eq:sum-con-closed}
\end{equation}

\paragraph{Step 4: telescoping and using $\gamma_t\ge \underline{\gamma}$.}
Summing~\eqref{eq:descent-ready} over $t=0,\dots,T-1$ telescopes the objective values:
\[
\sum_{t=0}^{T-1}\mathbb{E}\!\left[F(\bar w^t)-F(\bar w^{t+1})\right]
=
F(\bar w^0)-\mathbb{E}[F(\bar w^T)]
\le
F(\bar w^0)-F^\star.
\]
By Assumption~\ref{ass:mass-bounded}, $\gamma_t\ge \underline{\gamma}$, hence
\[
\sum_{t=0}^{T-1}\mathbb{E}\|\nabla F(\bar w^t)\|^2
\le
\frac{1}{\underline{\gamma}}
\sum_{t=0}^{T-1}\mathbb{E}\!\left[\gamma_t\|\nabla F(\bar w^t)\|^2\right].
\]
Plugging~\eqref{eq:sum-con-closed} into the sum of~\eqref{eq:descent-ready}, dividing by
$\alpha_0\eta E\,\underline{\gamma}\,T$ and absorbing fixed numerical factors into the explicit coefficients
yields~\eqref{eq:main-closed-1}-\eqref{eq:main-closed-2} with $D_\Delta$ and $B_{\mathrm{con}}$ given in
\eqref{eq:Ddelta-closed}-\eqref{eq:Bcon-closed}.
\end{proof}

\section{Supplementary Experiments}
\label{app:supplementary_experiments}
\subsection{Accuracy Curves}
\label{app:learn_curves}
To provide a more complete view of the learning dynamics beyond the final averaged accuracy, we report accuracy curves for all datasets. Specifically, we include (i) \emph{global} test accuracy trajectories across communication rounds and (ii) accuracy trajectories of \emph{delayed clients} measured on a pseudo-time axis that reflects asynchronous progress. These curves complement the main results by visualizing convergence speed, stability and robustness under system asynchrony.

\subsubsection{Global Accuracy Curves}
\label{app:global_learn_curves}
This subsection reports the evolution of global test accuracy across communication rounds.
Figures~\ref{fig:app:cifar10_global_acc}-\ref{fig:app:tiny_global_acc} show the global accuracy curves on CIFAR-10, CIFAR-100 and Tiny-ImageNet, respectively.
These curves illustrate the overall convergence behavior of different methods, including their convergence speed and final performance, under identical experimental settings.

\begin{figure}[t]
    \centering
    \begin{subfigure}[t]{0.45\linewidth}
        \centering
        \includegraphics[width=\linewidth]{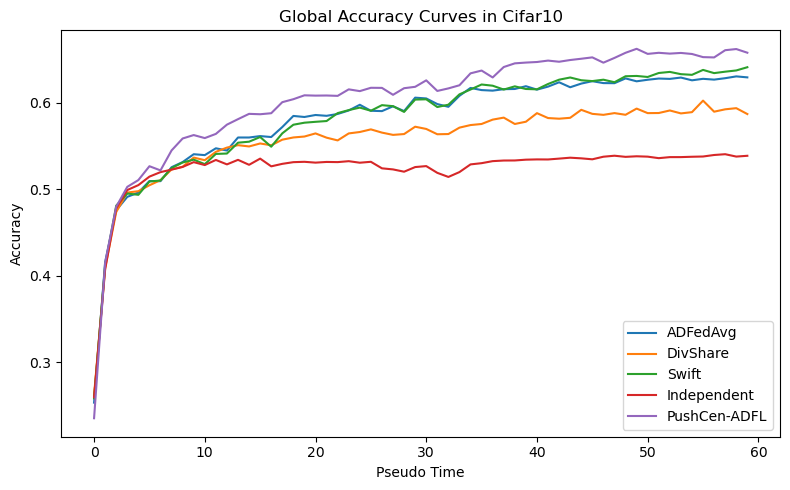}
        \Description{Global test accuracy curves on CIFAR-10.}
        \caption{CIFAR-10}
        \label{fig:app:cifar10_global_acc}
    \end{subfigure}
    \begin{subfigure}[t]{0.45\linewidth}
        \centering
        \includegraphics[width=\linewidth]{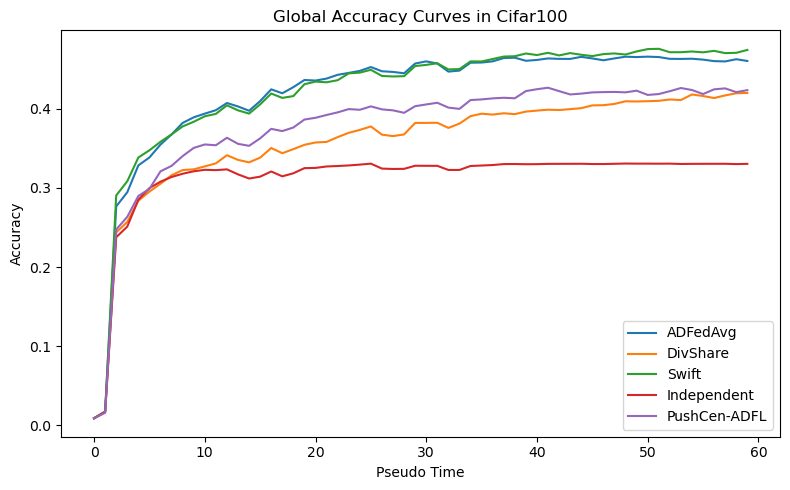}
        \Description{Global test accuracy curves on CIFAR-100.}
        \caption{CIFAR-100}
        \label{fig:app:cifar100_global_acc}
    \end{subfigure}

    \vspace{0.6em}

    \begin{subfigure}[t]{0.45\linewidth}
        \centering
        \includegraphics[width=\linewidth]{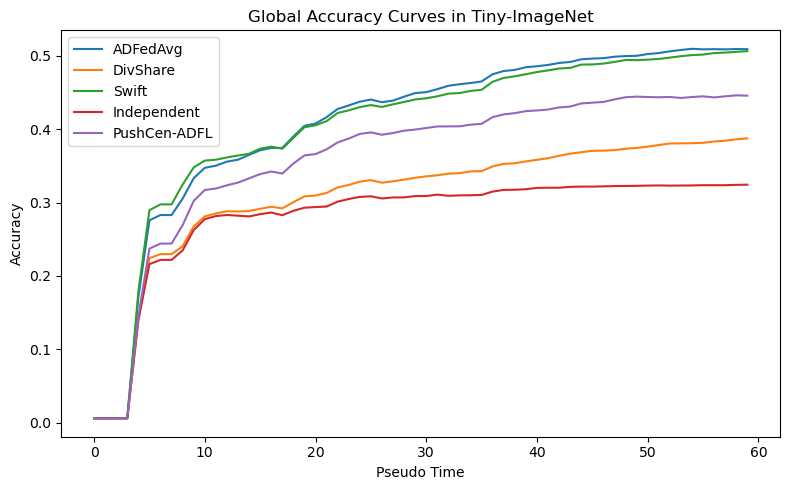}
        \Description{Global test accuracy curves on Tiny-ImageNet.}
        \caption{Tiny-ImageNet}
        \label{fig:app:tiny_global_acc}
    \end{subfigure}

    \caption{Global test accuracy curves on CIFAR-10, CIFAR-100, and Tiny-ImageNet.}
    \label{fig:app:global_acc_all}
\end{figure}

\subsubsection{Delayed Client Accuracy Curves}
\label{app:delayed_client_learn_curves}
This subsection presents accuracy curves for delayed clients, evaluated along a pseudo-time axis that reflects asynchronous local progress.
Figures~\ref{fig:app:cifar10_delayed_avg_acc}, \ref{fig:app:cifar100_delayed_avg_acc} and \ref{fig:app:tiny_delayed_avg_acc} report the average test accuracy across delayed clients on CIFAR-10, CIFAR-100 and Tiny-ImageNet, respectively.
In addition, Figures~\ref{fig:app:cifar10_delayed_clients}, \ref{fig:app:cifar100_delayed_clients} and \ref{fig:app:tiny_delayed_clients} visualize the accuracy trajectories of representative individual delayed clients, providing a fine-grained view of training stability under client delays.

\begin{figure*}
    \centering
    \includegraphics[width=0.80\linewidth]{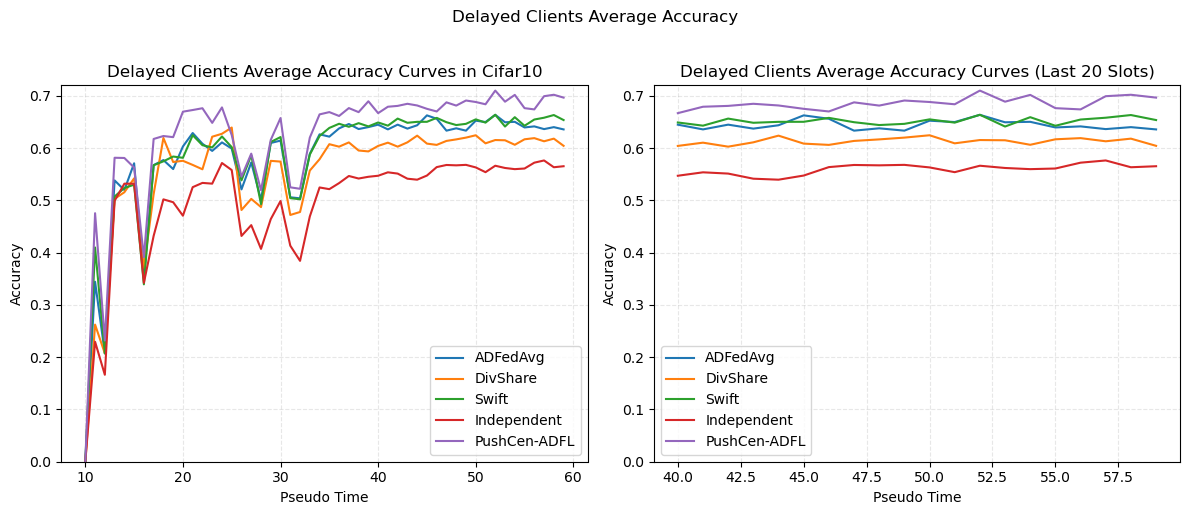}
    \Description{Average test accuracy of delayed clients on Cifar10.}
    \caption{Average test accuracy of delayed clients on Cifar10.}
    \label{fig:app:cifar10_delayed_avg_acc}
\end{figure*}

\begin{figure*}
    \centering
    \includegraphics[width=0.80\textwidth]{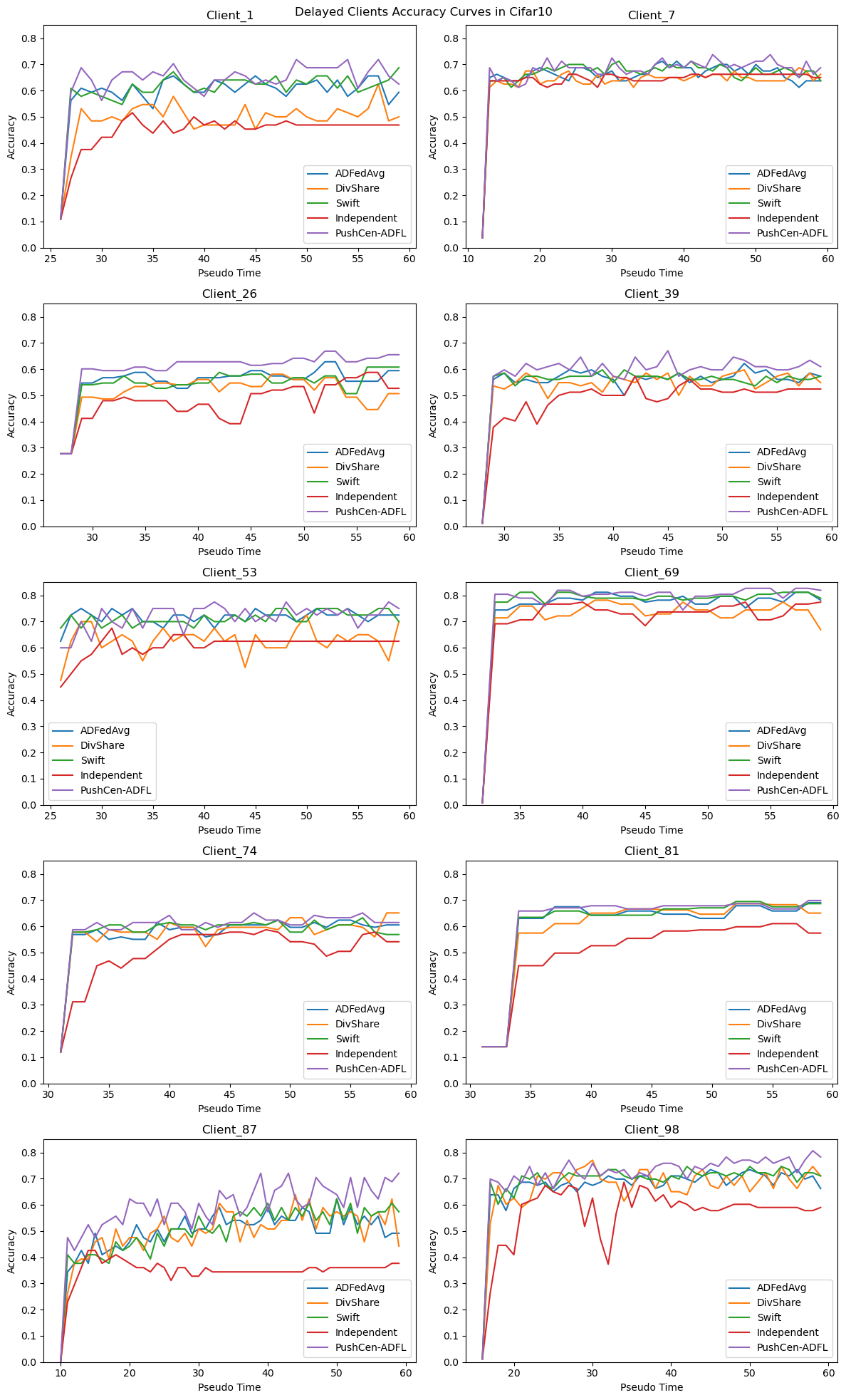}
    \Description{Accuracy curves of representative delayed clients on CIFAR-10 (pseudo-time axis).}
    \caption{Accuracy curves of representative delayed clients on CIFAR-10 (pseudo-time axis).}
    \label{fig:app:cifar10_delayed_clients}
\end{figure*}

\begin{figure*}
    \centering
    \includegraphics[width=0.80\linewidth]{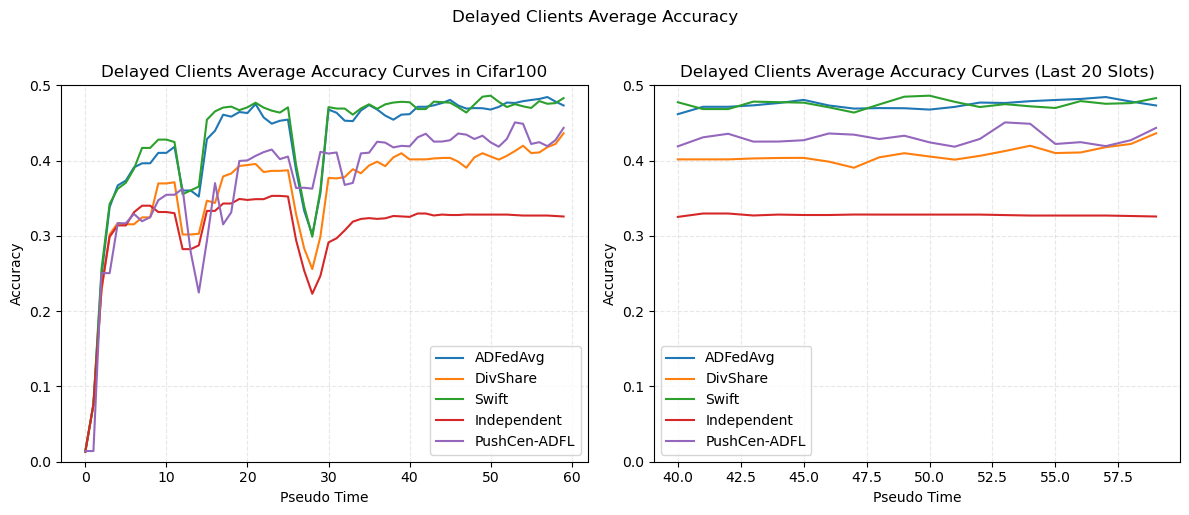}
    \Description{Average test accuracy of delayed clients on Cifar100.}
    \caption{Average test accuracy of delayed clients on Cifar100.}
    \label{fig:app:cifar100_delayed_avg_acc}
\end{figure*}

\begin{figure*}
    \centering
    \includegraphics[width=0.80\textwidth]{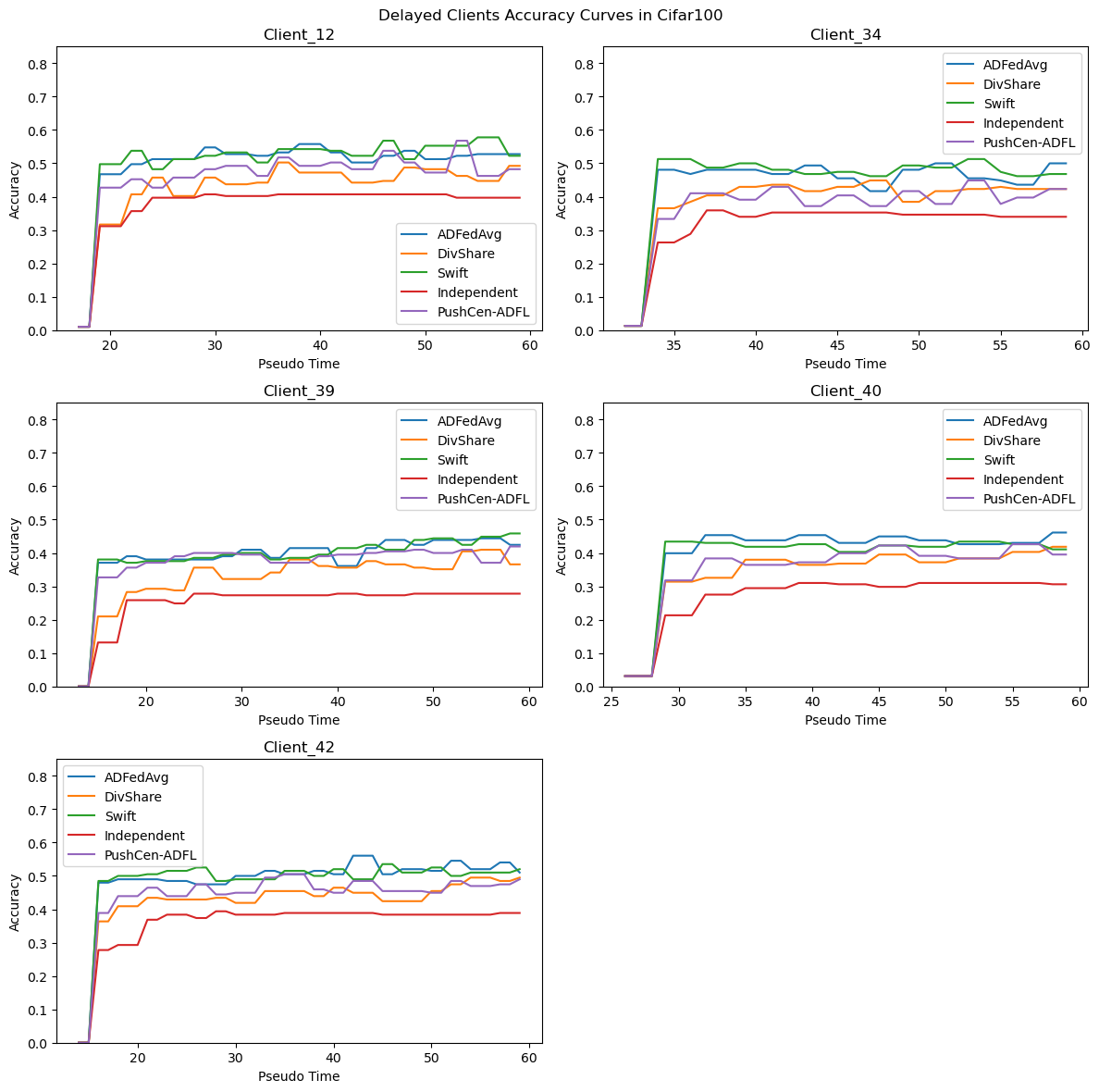}
    \Description{Accuracy curves of representative delayed clients on CIFAR-100.}
    \caption{Accuracy curves of representative delayed clients on CIFAR-100.}
    \label{fig:app:cifar100_delayed_clients}
\end{figure*}

\begin{figure*}
    \centering
    \includegraphics[width=0.80\linewidth]{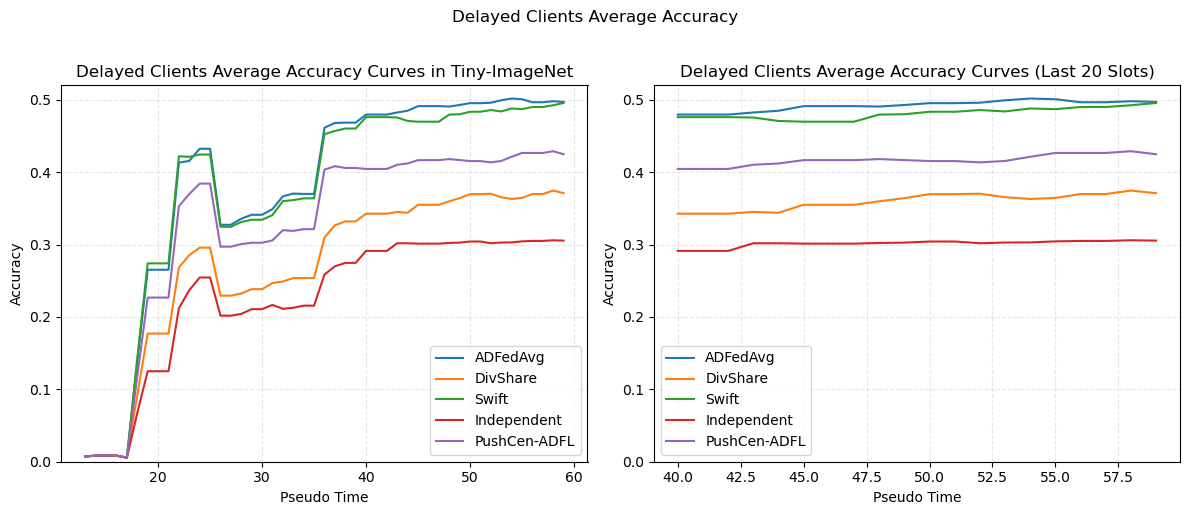}
    \Description{Average test accuracy of delayed clients on Tiny-ImageNet.}
    \caption{Average test accuracy of delayed clients on Tiny-ImageNet.}
    \label{fig:app:tiny_delayed_avg_acc}
\end{figure*}

\begin{figure*}
    \centering
    \includegraphics[width=0.80\textwidth]{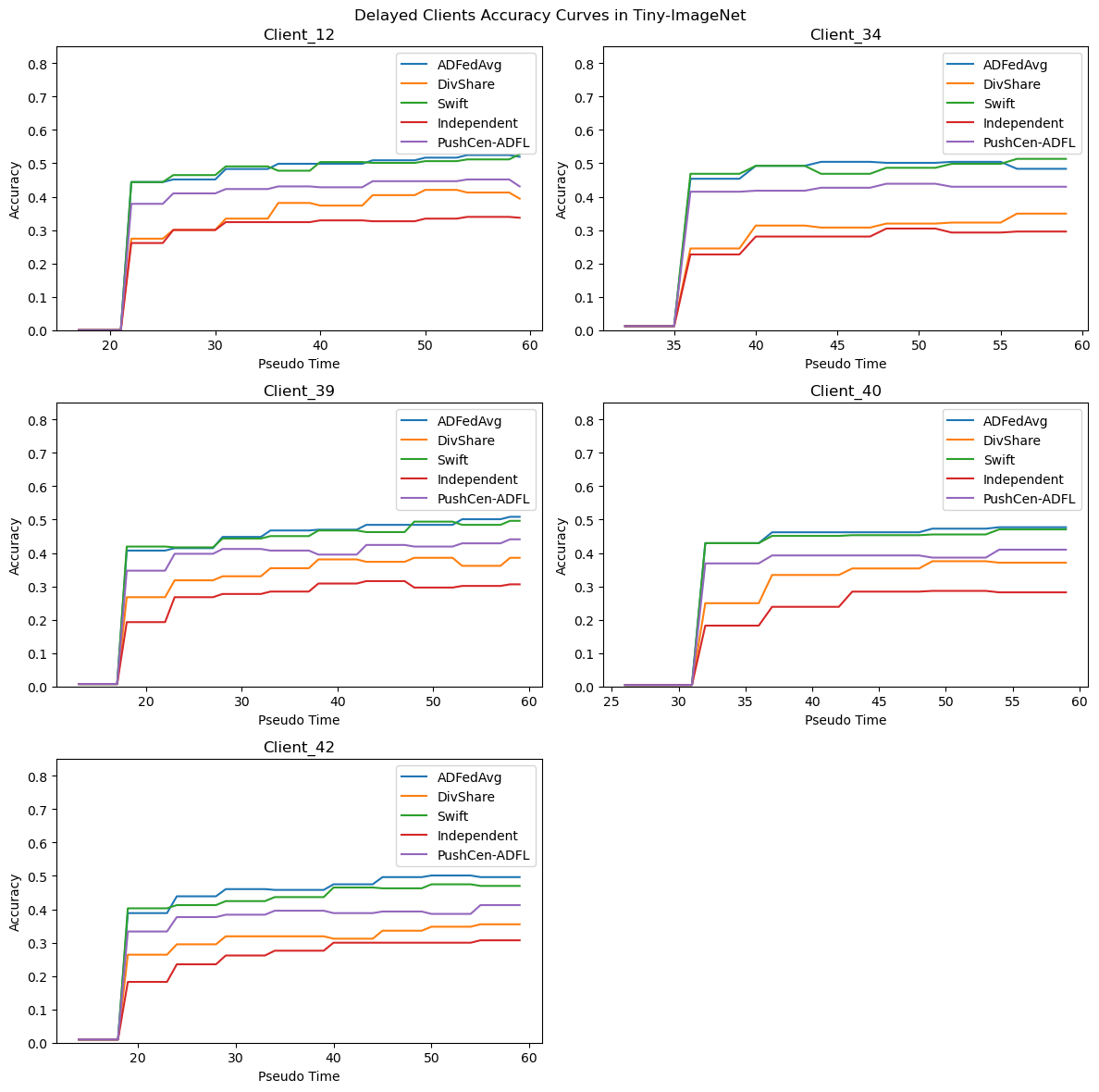}
    \Description{Accuracy curves of representative delayed clients on Tiny-ImageNet.}
    \caption{Accuracy curves of representative delayed clients on Tiny-ImageNet.}
    \label{fig:app:tiny_delayed_clients}
\end{figure*}

\end{document}